\documentclass{article}

\usepackage{microtype}
\usepackage{graphicx}
\usepackage{subfigure}
\usepackage{booktabs} %
\usepackage{placeins}
\usepackage{hyperref}

\usepackage[accepted]{icml2024}

\usepackage{amsmath}
\usepackage{amssymb}
\usepackage{mathtools}
\usepackage{amsthm}

\usepackage[symbol]{footmisc}

\usepackage[capitalize,noabbrev]{cleveref}

\theoremstyle{plain}
\newtheorem{theorem}{Theorem}
\newtheorem{proposition}[theorem]{Proposition}

\theoremstyle{definition}

\theoremstyle{remark}

\newcommand{\approachfull}{Chunked Temporal Difference}

\newcommand{\approachabbr}{Chunked-TD}

\icmltitlerunning{Sequence Compression Speeds Up Credit Assignment in Reinforcement Learning}

\begin{document}

\twocolumn[
\icmltitle{Sequence Compression Speeds Up \\ Credit Assignment in Reinforcement Learning}

\icmlsetsymbol{equal}{*}

\begin{icmlauthorlist}
\icmlauthor{Aditya A. Ramesh}{usi}
\icmlauthor{Kenny Young}{uofa}
\icmlauthor{Louis Kirsch}{usi}
\icmlauthor{J{\"u}rgen Schmidhuber}{usi,kaust}
\end{icmlauthorlist}

\icmlaffiliation{usi}{The Swiss AI Lab IDSIA, USI \& SUPSI}
\icmlaffiliation{uofa}{University of Alberta and the Alberta Machine Intelligence Institute}
\icmlaffiliation{kaust}{AI Initiative, King Abdullah University of Science and Technology}

\icmlcorrespondingauthor{Aditya A. Ramesh}{aditya@idsia.ch}

\icmlkeywords{Reinforcement Learning,
Credit assignment,
Temporal Difference,
Chunking,
ICML}

\vskip 0.3in
]

\printAffiliationsAndNotice{}  %

\begin{abstract}
Temporal credit assignment in reinforcement learning is challenging due to delayed and stochastic outcomes. 
Monte Carlo targets can bridge long delays between action and consequence but lead to high-variance targets due to stochasticity. 
Temporal difference (TD) learning uses bootstrapping to overcome variance but introduces a bias that can only be corrected through many iterations.
TD($\lambda$) provides a mechanism to navigate this bias-variance tradeoff smoothly.
Appropriately selecting $\lambda$ can significantly improve performance.
Here, we propose \emph{\approachabbr}, which uses predicted probabilities of transitions from a model for computing $\lambda$-return targets. 
Unlike other model-based solutions to credit assignment, \approachabbr \ is less vulnerable to model inaccuracies.
Our approach is motivated by the principle of history compression and `chunks' trajectories for conventional TD learning.
Chunking with learned world models compresses near-deterministic regions of the environment-policy interaction to speed up credit assignment while still bootstrapping when necessary.
We propose algorithms that can be implemented online and show that they solve some problems much faster than conventional TD($\lambda$).
\end{abstract}

\section{Introduction}
\label{intro}

An intelligent agent should recognize which of its actions contributed to success or failure~\citep{minsky1961steps}. 
Assigning credit or blame can be particularly difficult when there are long delays between a critical action and its consequences.
Stochastic events between the action and its eventual consequence further complicate credit assignment.
For example, consider the scenario where taking the umbrella in the morning helped the agent stay dry in the evening rain.

Reinforcement learning (RL) agents can bridge long delays by using Monte Carlo (MC) returns to evaluate---and subsequently improve---their behavior.
However, using MC returns comes at the cost of high variance when the environment or the agent's behavior is stochastic.
Temporal-difference (TD) approaches are designed to counter the variance of MC returns by constructing targets using predicted outcomes in future states (bootstrapping). 
TD lowers variance at the cost of introducing bias.
TD($\lambda$) approaches provide a way to smoothly interpolate between one-step bootstrapping (TD(0)) and MC targets by selecting $\lambda$.
Unfortunately, correcting the bias of TD approaches can be very slow due to long delays between action and consequence~\citep{arjona2019rudder}.
TD with underparameterized function approximation introduces additional bias, even asymptotically~\citep{sutton2018reinforcement}.
Navigating this bias-variance trade-off is critical to achieving accurate and sample-efficient credit assignment~\citep{watkins1989learning, kearns2000bias}.
Previous research has attempted to automate the selection of $\lambda$ to improve the performance of TD algorithms~\citep{sutton1994step, white2016greedy, xu2018meta}.

An alternative to the previously discussed model-free approaches is to build a model of the world and use it for credit assignment.
A world model can make associations between cause and effect, which could be used to guide policy learning.
A prominent approach in this direction is to learn a differentiable model of the world and use gradients for credit assignment~\citep{Werbos:87, Munro:87, Schmidhuber:90diffenglish, heess2015learning}.
Such approaches have successfully been applied to continuous control problems with smooth dynamics~\citep{deisenroth2011pilco, amos2021model}.
Unfortunately, this idea is not easily applicable in environments with discrete actions and is quite sensitive to inaccuracies in the learned model~\citep{DBLP:conf/iclr/HafnerL0B21, meulemans2023would}.

World models can also be used to simulate experience to train the agent's policy or value function~\citep{Schmidhuber:90diffenglish,  Sutton:90dyna, ha2018world, DBLP:conf/iclr/HafnerL0B21}.
Imagination with world models (or experience replay) can propagate credit without requiring the agent to re-experience the same transitions in the true environment.
This can speed up credit assignment in terms of interactions needed with the true environment.
However, imagination-based approaches can drastically fail when the learned world model is inaccurate~\citep{talvitie2017self, abbas2020selective}.

In this paper, we propose an approach that uses (possibly inaccurate) learned models to dynamically chunk trajectories of experience for TD learning.
We introduce \emph{\approachfull \ (\approachabbr)} algorithms that compute adaptive $\lambda$-returns using model predictions for on-policy value learning.
Our approach is inspired by the principle of history compression~\citep{Schmidhuber:91chunker, Schmidhuber:92ncchunker}, which uses predictive coding to remove redundant information from a sequence to shorten credit assignment paths.
Chunking in RL can shorten credit assignment paths in deterministic and predictable regions of the environment-policy interaction and facilitate faster learning while also bootstrapping to handle noise where necessary.
Our proposed algorithms can be implemented online using eligibility traces~\citep{sutton1984temporal}.
Importantly, since the model is used only to decide when to bootstrap, our approach is robust to arbitrary levels of model error and simply degrades toward TD(0) as model errors become extreme.

Our experiments show that Chunked-TD can assign credit much faster than conventional TD($\lambda$) in the considered Markov decision processes with tabular value functions.
We also present a variant of Chunked-TD that can use factored rewards to solve a typical hard credit assignment problem.
\footnote[4]{Code is available at \url{https://github.com/Aditya-Ramesh-10/chunktd}.}

Our approach is closely related to a previously proposed variable TD($\lambda$) strategy~\citep{sutton1994step} which is motivated from the perspective of reducing the bias of TD approaches.
Similar to our approach, it uses transition probabilities from tabular models to decay traces.
As a separate contribution, we revisit their proposed algorithms and show that their approach can be viewed as using `backward' model predictions for constructing $\lambda$-returns.

\section{Preliminaries}
\label{prelims}

\subsection{RL notation and definitions}
\label{prelims:rl_notation}

Consider an agent interacting with a Markov Decision Process (MDP)~\citep{puterman1990markov} with finite state, action, and reward spaces ($\mathcal{S}, \mathcal{A},$ and $\mathcal{R}$) and a discount factor $\gamma \in \left [  0, 1 \right ]$. 
The discount factor can be one in episodic tasks.
The agent interacts with the environment at discrete time steps through actions sampled from a policy $A_t \sim \pi(. | S_t )$, where the policy $\pi : \mathcal{S} \rightarrow \Delta \left ( \mathcal{A} \right )$ maps states to a distribution over actions.
In response to the agent's action, the environment produces a new state and reward according to transition dynamics $S_{t+1} \sim T(.| S_t, A_t)$ and $R_{t+1} \sim R(.| S_t, A_t)$.

Interaction between policy and MDP generates a sequence $$ R_0, S_0, A_0 , R_1, S_1, A_1 \dots S_{t-1}, A_{t-1}, R_t, S_t, A_t \dots ,$$
where $R_0 = 0$ (for convenience and uniformity).
Let $X_t = (R_t, S_t)$ be the `concatenation' of reward and state. 
We will refer to $X_t$ as the percept at time step $t$. 
It represents the response of the environment as a single object. 
In general, we could use $X_t = f(R_t, S_t)$ with some suitable function for state abstraction.

The above interaction can then be re-expressed as $$ X_0, A_0 , X_1, A_1 \dots X_{t-1}, A_{t-1}, X_t \dots.$$

The agent's goal is to maximize the expected discounted return, $ \mathbb{E} \left [\sum_{t=0}^{\infty} \gamma^t R_{t+1} \right ]$, where the expectation accounts for randomness in the environment and the agent's policy.

RL algorithms often learn estimates of state or state-action values as an intermediate step for improvement.
The value of a state $V_{\pi}(s)$ is defined as the expected sum of discounted rewards when starting from $s$ and acting according to $\pi$,  $V_{\pi}(s) = \mathbb{E} \left [ \sum_{t=0}^{\infty} \gamma^t R_{t+1} \right  | S_0 = s]$.
Similarly, action-value functions can be defined as $Q_{\pi}(s,a) = \mathbb{E} \left [ \sum_{t=0}^{\infty} \gamma^t R_{t+1} \right  | S_0 = s, A_0 = a]$. 

Targets for value learning can be obtained via Monte-Carlo estimates or bootstrapped targets through temporal difference.
Interpolating smoothly between MC and one-step bootstrapped targets is possible through the use of $\lambda$-returns~\citep{watkins1989learning, sutton2018reinforcement}. 
The $\lambda$-return ($G_t^{\lambda}$) can be recursively defined as
\begin{align}
\label{eqn:const_lambda_return}
    G_{t}^{\lambda} = R_{t+1} + \gamma \lambda G_{t+1}^{\lambda} + (1 - \lambda ) \gamma \hat{V}(S_{t+1}),
\end{align}
where $\hat{V}(S_{t+1})$ is the current estimate of the value of $S_{t+1}$.
Using $\lambda=1$ amounts to the MC target and $\lambda=0$ uses one-step bootstrapped TD targets.

One can use $\lambda$-return targets to update value functions in the following manner,
\begin{align}
\label{eqn:offline-lambda-update}
    \hat{V} (S_t) \leftarrow \hat{V} (S_t) + \alpha \left ( G_{t}^{\lambda} - \hat{V}(S_t) \right ),
\end{align}
where $\alpha$ is the learning rate.
Similar updates can be used for learning $\hat{Q}(S_t, A_t)$.

Updates with $\lambda$-returns can be done in an offline manner upon completion of episodes.
This is because computing the $\lambda$-return according to Equation~\ref{eqn:const_lambda_return} requires knowing the entire sequence of future states and rewards.
However, it is possible to approximate updates from offline $\lambda$-returns in an online, incremental way through TD($\lambda$), which uses eligibility traces~\citep{sutton1984temporal}.

\begin{figure}[t]
    \centering
    \includegraphics[width=\columnwidth]{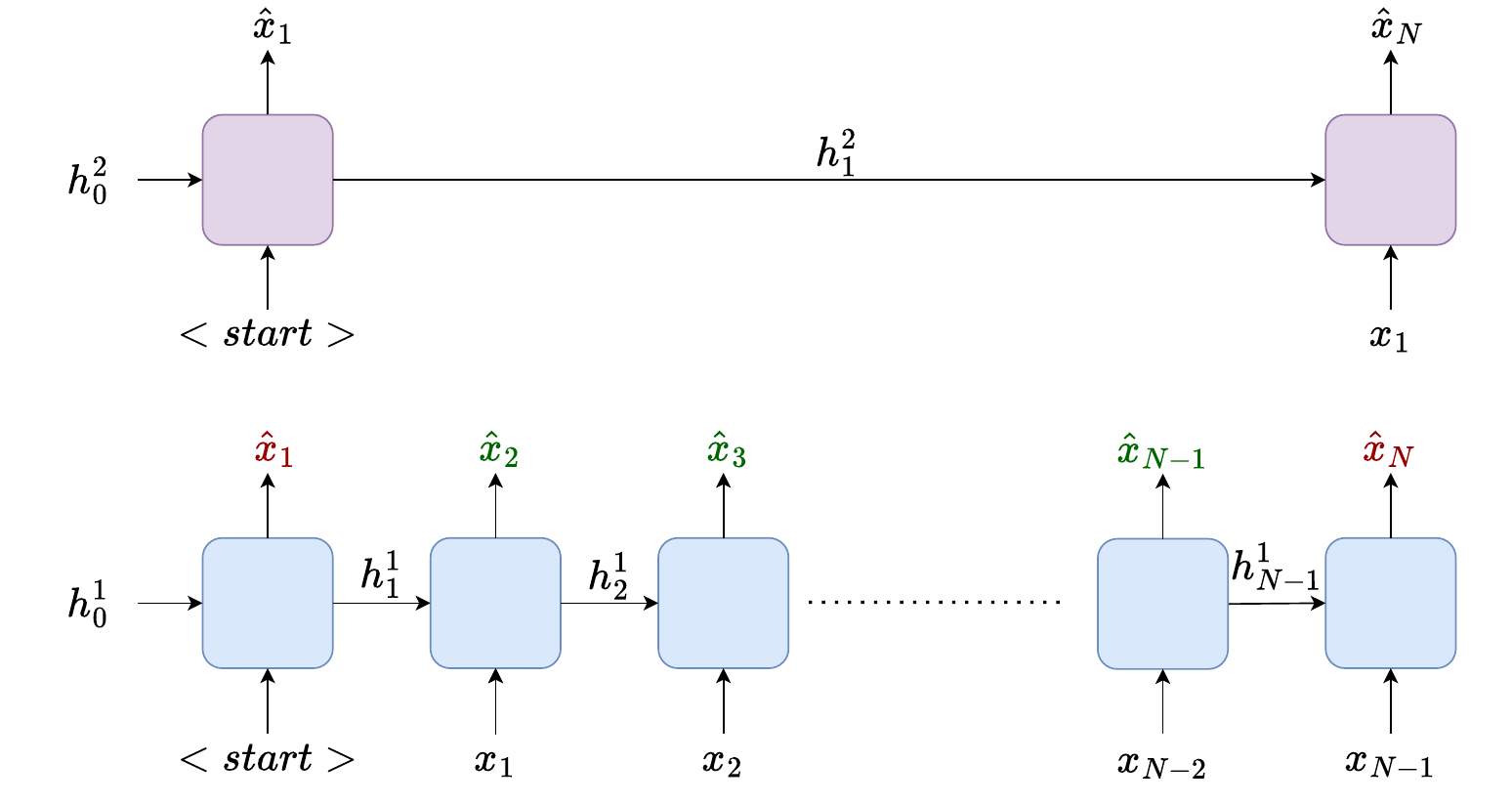}
    \caption{Two-level Chunker-RNNs can shorten credit assignment paths. Only the tokens associated with incorrect predictions (denoted with red) at the lower level are passed on to the higher level RNN (with the start token).
    The hidden state of the RNN at level $l$ and index $i$ is denoted as $h^l_i$.
    }
    \label{fig:chunker_rnn}
\end{figure}

\subsection{Chunking and the principle of history compression}
\label{prelims:chunker}

The principle of history compression~\citep{Schmidhuber:91chunker,Schmidhuber:92ncchunker} suggests removing redundant information from a sequence to shorten credit assignment paths.
The aforementioned paper introduces a multi-level hierarchy of recurrent neural networks (RNNs) for sequence modeling or predictive coding, i.e., predicting the next input token. 
The central idea is that incorrectly predicted next inputs are identified based on a threshold on the error and passed to a higher level RNN, which can relate its inputs with shorter credit assignment paths.

We illustrate the idea through the example presented by \citet{Schmidhuber:92ncchunker} where there are two distinct sequences $<start>, a, i_1, i_2, \dots i_{N-1}, 0$ and $<start>, b, i_1, i_2, \dots i_{N-1}, 1$.
Depending on whether the sequence starts with an $a$ or $b$, the final token is $0$ or $1$.
A standard RNN predicting the next input token finds it challenging to model this sequence due to the long delay between the critical input ($a$ or $b$) and the classification label (0 or 1).
However, the architecture with two RNNs, as presented in Figure \ref{fig:chunker_rnn}, easily solves this problem.

At every time step $t$, each RNN tries to predict its next input token ($x_{t+1}$). 
Since the sub-sequence $i_1, i_2 \dots i_{N-1}$ is highly predictable from local context, the lower level RNN easily learns to predict these tokens.
The `unpredictable' and compressed sequence passed to the higher level would be $<start>, a,  0$ or $<start> ,b,  1$, which removes the delay in credit assignment.

In this paper, we take inspiration from the idea of history compression.
While we do not use a hierarchical policy or value function, 
we employ the idea of `compressing' sequences through next token prediction with a learned model.
This can help address some of the difficulties arising from long gaps between crucial action and rewarding consequences in RL.
In our case, the sequences are trajectories or episodes of interaction in reinforcement learning settings.
As we explain in the next section, we use the compressed sequence to construct targets for value estimation.

\section{Method}
\label{method}

First, in Section~\ref{sec:method:subsec:nstep}, we present an approach that uses $n$-step targets for learning value functions based on a dynamic (and stochastic) choice of $n$ obtained through sequence compression/chunking. 

Next, in Section~\ref{method:chunked-td-lambda}, we reformulate the stochastic chunking procedure as an exact, expected update, which is expressed as a $\lambda$-return.
We show that our update can also be used to implement online algorithms through eligibility traces.

\subsection{$n$-step chunked targets}
\label{sec:method:subsec:nstep}

Long delays between an action and its rewarding consequence make credit assignment in RL difficult~\citep{arjona2019rudder}.
As discussed in Section~\ref{prelims:chunker}, 
removing redundant information from a sequence can shorten credit assignment paths for sequence modeling~\citep{Schmidhuber:92ncchunker}. 
With a similar motivation, we aim to use a learned (generative) model of the environment to shorten credit assignment paths in RL when possible.
We use the predicted probability of the next input under the learned model to chunk an episode of policy-environment interaction.
The chunked episode will provide us with the states from which we bootstrap. 

In our first variant, which applies to bootstrapping from state values, we use $\hat{P}^{\pi}(X_{t+1} | X_t)$, the probability of the next percept under the current policy. The following factorization can be used in MDPs with Markov policies,
\begin{equation}
\hat{P}^{\pi}(X_{t+1} | X_t) = \sum_{a \in \mathcal{A}} \hat{P}(X_{t+1} | X_t, a) \pi( a | X_t), 
\end{equation}
where $\hat{P}(X_{t+1} | X_t, a)$ is the estimated transition probability given a particular action.
Note that while policies are typically defined as a function of states, they can also be defined on percepts directly by constraining policy functions to ignore the reward component of the percept.

Consider the following chunking strategy that applies to learning $\hat{V}(s)$ in episodic MDPs: (1) We always keep the first state $s_0$ (2) For $t \geq 1$ we drop state $s_t$ with probability $\hat{P}^{\pi}(x_{t+1} | x_t )$ (3) Whenever we drop a percept, we sum the intermediate rewards such that the compressed and original episode have equivalent returns (4) We always keep the terminal/ultimate state $s_{T}$.

The compressed episode is used to construct targets for the states encountered in the entire episode.
Consider the example presented in Figure~\ref{fig:sampled-chunked-td-valuel}, where an episode of interaction produces percepts $x_0, x_1, \dots x_{T}$.
All transitions are deterministic and accurately modeled, apart from one transition $x_k \rightarrow x_{k+1}$ which has a low probability.
In this situation, our chunking procedure may return the compressed state sequence of $x_o, x_k, x_T$, which are used for bootstrapping.
All states before $x_k$ using the value of $\hat{V}(s_k)$ to bootstrap.
All states thereafter use Monte-Carlo targets (bootstrapping from the terminal state $\hat{V}(s_T)$ which has a value of zero).

Thus, we can use a learned model to chunk and obtain $n$-step bootstrapped backups with a stochastic $n$. 
The key feature of this approach is that if the environment-policy interaction is quite deterministic and our model is accurate, then it is like Monte Carlo learning. 
If there is high stochasticity, \emph{or if our model is inaccurate}, our approach falls back to bootstrapping, with TD(0) being the extreme case.
A similar approach can be used for bootstrapping from Q-value functions, where we chunk the episode using $\hat{P}^{\pi}(X_{t+1}, A_{t+1} | X_t, A_t)$ (see Appendix~\ref{app:algorithm_variants}).

\begin{figure*}[t]
    \centering
    \includegraphics[width=0.9\textwidth]{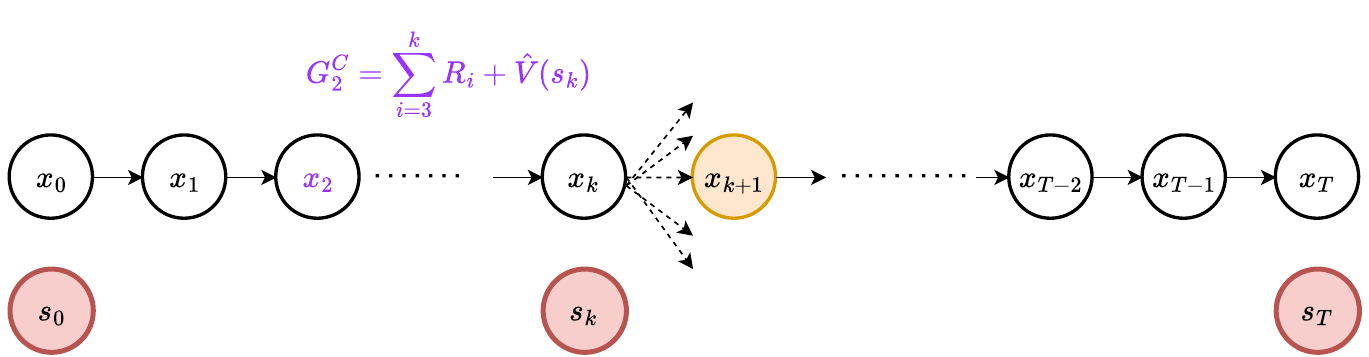}
    \caption{Sample-based chunking (see Section \ref{sec:method:subsec:nstep}) for state value estimation. All next states (and rewards) are deterministic apart from the outcome at $x_k$, where $P_{\pi}(x_{k+1} | x_k)$ is low. The states in red are a sample from our algorithm, i.e., chunking based on model probabilities that provide targets for bootstrapping. 
    The chunked target for $s_2$ (with $\gamma=1$, denoted as $G_2^C$) is shown as an example.
    }
    \label{fig:sampled-chunked-td-valuel}
\end{figure*}

\subsection{Chunked-TD($\lambda$)}
\label{method:chunked-td-lambda}

The sampling-based chunking of the previous section can be reformulated as an expected update without the variance of sampling indicator variables to drop or keep states. 

We recall the recursive definition of the $\lambda$-return ($G_{t}^{\lambda}$) with a time-varying $\lambda_{t+1}$, 
\begin{equation}
\label{eqn:variable_lambda_return}
    G_{t}^{\lambda} = R_{t+1} + \gamma \lambda_{t+1} G_{t+1}^{\lambda} + (1 - \lambda_{t+1}) \gamma V(S_{t+1}).
\end{equation}

As motivated in Section~\ref{sec:method:subsec:nstep}, we bootstrap based on predictions from a forward model.
Concretely, we keep $S_{t+1}$ in the compressed sequence with probability $1 - \hat{P}^{\pi}(x_{t+2}| x_{t+1})$.
Keeping it in the compressed sequence is equivalent to using it for bootstrapping, i.e., using $\lambda_{t+1} = 0$ in Equation~\ref{eqn:variable_lambda_return}.

It can be seen that using $\lambda_{t+1} = \hat{P}^{\pi}(x_{t+2}| x_{t+1})$ is equivalent to the expected update from the chunking procedure described in the previous section.

\paragraph{Eligibility-trace based online algorithm}

\begin{algorithm}[tb]
\caption{Chunked-TD state value evaluation}
\label{alg:chunked_td_v_v}
\begin{algorithmic}
\ENSURE Eligibility $e(s) = 0 \ \forall s \in \mathcal{S}$ at the beginning of each episode
\FOR{each step of the episode $(X_t, A_t, X_{t+1})$}
\STATE Optionally train dynamics model

\FOR{each state $s$}
\STATE $e(s) \gets \gamma  \hat{P}^{\pi}( X_{t+1} | X_t ) e(s)$ %
\ENDFOR
\STATE $e({S_t}) \gets e({S_t}) + 1$ 
\STATE $\delta_t \gets  R_{t+1} + \gamma \hat{V}(S_{t+1}) - \hat{V}(S_t)$
\FOR{each state $s$}
\STATE $\hat{V}(s) \gets \hat{V}(s) + \alpha \delta_t e(s)$
\ENDFOR
\ENDFOR
\end{algorithmic}
\end{algorithm}

Since the target for $S_t$ ($G_t^{\lambda}$) depends on $S_{t+1}$ and $S_{t+2}$, it may not be immediately obvious that updates based on such a $\lambda$-return can be implemented online.
We show that online updates are indeed possible. Algorithm \ref{alg:chunked_td_v_v} presents an incremental and online algorithm that approximates updates based on the $\lambda$-return described in Equation \ref{eqn:variable_lambda_return} with $\lambda_{t+1} = \hat{P}^{\pi}(x_{t+2}| x_{t+1})$.
The equivalence is exact at the end of episodes in acyclic MDPs with tabular value functions.

\begin{proposition}
    Let $M$ be an acyclic episodic Markov decision process with state space $\mathcal{S}$ and let $\hat{V}: \mathcal{S} \rightarrow \mathbb{R}$ be the estimated tabular value function.
    Let the sequence $x_0, a_0, x_1, a_1 \dots x_T$ correspond to an episode of interaction.
    Then updates from offline-$\lambda$ returns from Equations~\ref{eqn:offline-lambda-update} and \ref{eqn:variable_lambda_return} with $\lambda_{t+1} = \hat{P}^{\pi}(x_{t+2}| x_{t+1})$ match the total updates made by Algorithm~\ref{alg:chunked_td_v_v} at the end of the episode.
\end{proposition}
\begin{proof} Proof sketch only.
We know that a state in the sequence, $s_t$, is encountered exactly once at step $t$ since $M$ is acyclic.
Without loss of generality, we consider a single state $s_t$.

Consider the update term $u_t$ added to $\hat{V}(s_t)$ by the offline-$\lambda$ return of Equation \ref{eqn:offline-lambda-update},
$$
    u_t = \alpha (G_t^\lambda-\hat{V}(s_t)).
$$
Unrolling the recursive definition from Equation \ref{eqn:variable_lambda_return}, we get
$$
    u_t = \alpha (r_{t+1} + \gamma \lambda_{t+1} G_{t+1}^{\lambda} + (1 - \lambda_{t+1}) \gamma \hat{V}(s_{t+1}) -\hat{V}(s_t)).
$$
Add and subtract \textcolor{blue}{$\gamma \lambda_{t+1} \hat{V} (S_{t+1})$}, 
\begin{align*}
    u_t &= \alpha( r_{t+1} + (1-\lambda_{t+1}) \gamma \hat{V}(s_{t+1}) \textcolor{blue}{+ \gamma \lambda_{t+1} \hat{V}(s_{t+1})} \\
    & - \hat{V}(s_{t})
    + \gamma \lambda_{t+1} G_{t+1}^{\lambda} \textcolor{blue}{- \gamma \lambda_{t+1} \hat{V}(s_{t+1})}). \\
\end{align*}
Using $\delta_t = r_{t+1} + \gamma \hat{V}(s_{t+1}) - \hat{V}(s_t)$, and unrolling,
\begin{align*}
u_t &= \alpha(\delta_t + \gamma \lambda_{t+1}(G_{t+1}^{\lambda} - \hat{V}(s_{t+1})) \\
    &= \alpha(\delta_t + \gamma \lambda_{t+1}(\delta_{t+1} + \gamma \lambda_{t+2} (G_{t+2}^{\lambda} - \hat{V}(s_{t+2}) ))).    
\end{align*}

Unrolling on, and noting  $\lambda_{t+1} = \hat{P}^{\pi}(x_{t+2} | x_{t+1})$,
\begin{align}
\label{eqn:offline-chunked-update}
    u_t = \alpha \sum_{k=t}^{T-1} \gamma^{k-t} \delta_k \prod_{i=t+1}^{k} \hat{P}^{\pi} (x_{i+1} | x_i ).
\end{align}

Now, let us consider the updates made by Algorithm~\ref{alg:chunked_td_v_v}. Only TD errors encountered from $s_t$ onward contribute towards updating $\hat{V}(s_t)$. The first transition from $s_t$ adds $\alpha \delta_t$ to $\hat{V}(s_t)$ (as $e(s_t)$ is 1 when $s_t$ is first encountered). The next update adds $\alpha \gamma \delta_{t+1} \hat{P}^{\pi} (x_{t+2} | x_{t+1})$, the one after adds $\alpha \gamma^2 \delta_{t+2} \hat{P}^{\pi} (x_{t+3} | x_{t+2}) \hat{P}^{\pi} (x_{t+2} | x_{t+1})$. Summing all of them gives the same update as Equation~\ref{eqn:offline-chunked-update}.
\end{proof}

So far, we have presented Chunked-TD for state value estimation.
Chunked algorithms can similarly be defined for action-value functions, which can be used to implement SARSA-like algorithms~\citep{rummery1994line, van2009theoretical} for evaluation and control.
We present these variants in Algorithms~\ref{alg:chunked_td_q_q} and \ref{alg:chunked_td_q_v} (see Appendix~\ref{app:algorithm_variants}).

\paragraph{Backward model} 
An interesting alternative is to use a backward model for chunking or choosing $\lambda_{t+1}$.
There one would have to model $\hat{P}^{\pi}(X_{t}, A_t | X_{t+1}, A_{t+1})$ or $\hat{P}^{\pi}(X_{t} | X_{t+1})$.
Intuitively, this will bootstrap more from states that may be reached by many different paths, which in turn means their value is likely to be updated many times between each visit to a particular transition.
Conversely, if a state is only ever visited from a single predecessor, bootstrapping is unlikely to be helpful as the value of that state and its predecessor will always be updated together, even with MC updates.
In Appendix~\ref{app:sns94}, we show that, with careful analysis, a previously proposed algorithm by \citet{sutton1994step} can be interpreted as using a $\lambda$-return by chunking with a backward model.
\footnote[2]{While the trace-based TDC algorithm in \citet{sutton1994step} uses a forward model to cut traces, with adjustments to the learning rate the corresponding $\lambda$-return (or forward view) amounts to using a backward model. Please see Appendix~\ref{app:sns94} for further details.}
A crucial disadvantage of backward models is that they are policy-dependent~\citep{chelu2020forethought}, making them harder to learn, especially in control, i.e., learning the optimal policy.

\section{Experiments}

This section presents an empirical comparison of Chunked-TD and conventional TD($\lambda$), i.e., with a constant scalar $\lambda$, for learning tabular action-value functions.
We conduct experiments in MDPs with fully-observable states to avoid conflating memory and credit assignment~\citep{osband2019behaviour,
ni2023transformers}.

We note that in our experiments, we set $X_t = S_t$ for Chunked-TD as all randomness related to the reward is absorbed in the next state in our considered environments.

\subsection{Chain-and-Split}
\label{exps:cns}

Chain-and-Split (see Figure~\ref{fig:simple_env}) is an environment that requires adaptive values of $\lambda$ for efficient policy evaluation.
It is an episodic MDP in which only the first action influences transitions and rewards.
One action leads to a `delayed' deterministic reward, whereas the other actions lead to stochastic consequences.

\begin{figure}
    \centering
    \includegraphics[width=0.9\columnwidth]{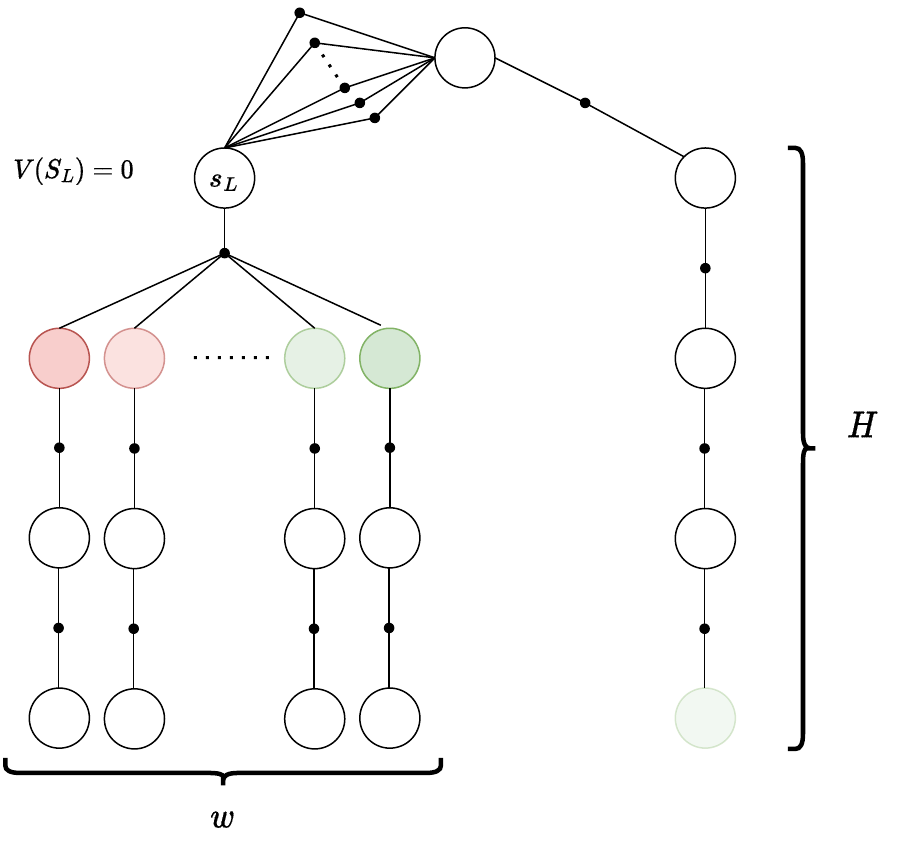}
    \caption{
    Chain-and-Split environment.
    States are denoted by large circles and actions by lines leading to the small black circles.
    Colored states have a reward upon entering that state.
    }
\label{fig:simple_env}
\end{figure}

\paragraph{Environment specification}

The agent has $n$ actions $\{ a_1, a_2, \dots a_n\}$ available  at the start/root state.
After the first step, there is only one action available in every state.
Taking $a_1$ at the start state leads to a linear chain of length $H$, which ultimately provides a deterministic reward of $0.01$.
This is the optimal choice.
All remaining actions lead to a common `parent' state on the left ($s_L$), which has a zero value but is followed by noisy transitions/rewards.
The state $s_L$ branches to one of $w$ states with equal probability. 
Each of the $w$ states provides a deterministic reward between -1 and 1, with the average over all states being equal to zero.
We set $H=20$, $n=10$, and $w=101$ for our experiment.

\paragraph{Evaluation details}

We compare the performance of SARSA(0), SARSA(1)/MC, and Chunked SARSA (Algorithm~\ref{alg:chunked_td_q_q}).
Tabular Q-values are learned from data collected by a policy that takes uniform random actions.
We use a learned tabular model for computing $\lambda_{t+1}$ for chunking. 
See Appendix~\ref{app:exps:cns} for further details.

We compare algorithms based on the action-value gap with respect to the optimal action. Concretely, we consider the quantity $\Delta Q = Q(start, a_1) - \max_{a \in \{  a_2, a_3 \dots a_n \} } Q(start, a)$, which is positive only when the optimal action is assigned a higher value than all others.
We consider multiple learning rates for each approach (see Appendix~\ref{app:exps:cns}).

\paragraph{Results}

\begin{figure}
    \centering
    \includegraphics[width=\columnwidth]{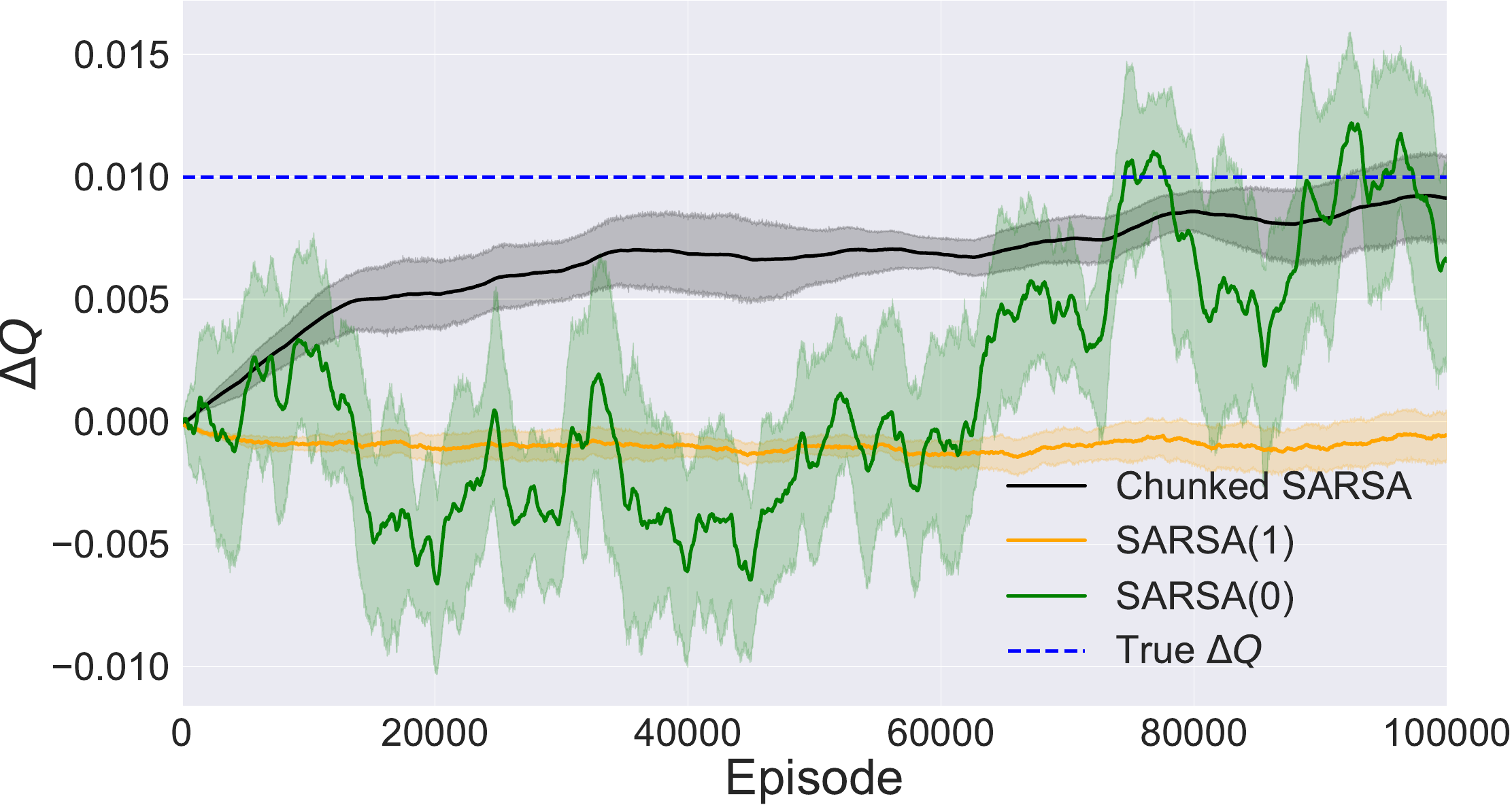}
    \caption{
    Results from the Chain-and-Split environment. The action-value gap ($\Delta Q$) between optimal action $a_1$ and the maximum action value among the remaining actions.
    The true value of $\Delta Q = 0.01$. 
    Shading indicates 95\%
    bootstrapped confidence intervals over 10 independent trials.
    }
    \label{fig:results_q_gap_simple_env}
\end{figure}

We present results with the best learning rate for each approach.
Figure~\ref{fig:results_q_gap_simple_env} shows that Chunked SARSA quickly discovers that taking $a_1$ is the best choice and smoothly approaches the true value of $\Delta Q$.

In contrast, MC/SARSA(1) fails because it does not bootstrap from the value of $s_L$.
This results in learning entirely independent action values $\hat{Q} (start, a_i), i \in \{ 2, 3 \dots 10\}$, without relying on the fact that all these actions actually lead to the same next state ($s_L$ in Figure~\ref{fig:simple_env}).
This results in negative values of $\Delta Q$, as the value of some action $ \hat{Q} (start, a_i), i \in \{ 2, 3 \dots 10\}$ is typically overestimated due to limited samples.

SARSA(0) bootstraps immediately, which is good for the left sub-tree as it correctly averages all the experience from sub-optimal actions.
However, SARSA(0) is slower and less stable than Chunked SARSA.
This is because SARSA(0) requires many experiences of taking $a_1$ to back up values along the linear chain on the right since the reward is provided only $H$ steps after selecting $a_1$ in the start state.

Additional results with longer evaluation
are reported in Appendix~\ref{app:exps:cns}.

\subsection{Accumulated-Charge}
\label{exps:acc_charge}

\begin{figure}
    \centering
    \includegraphics[width=0.99\columnwidth]{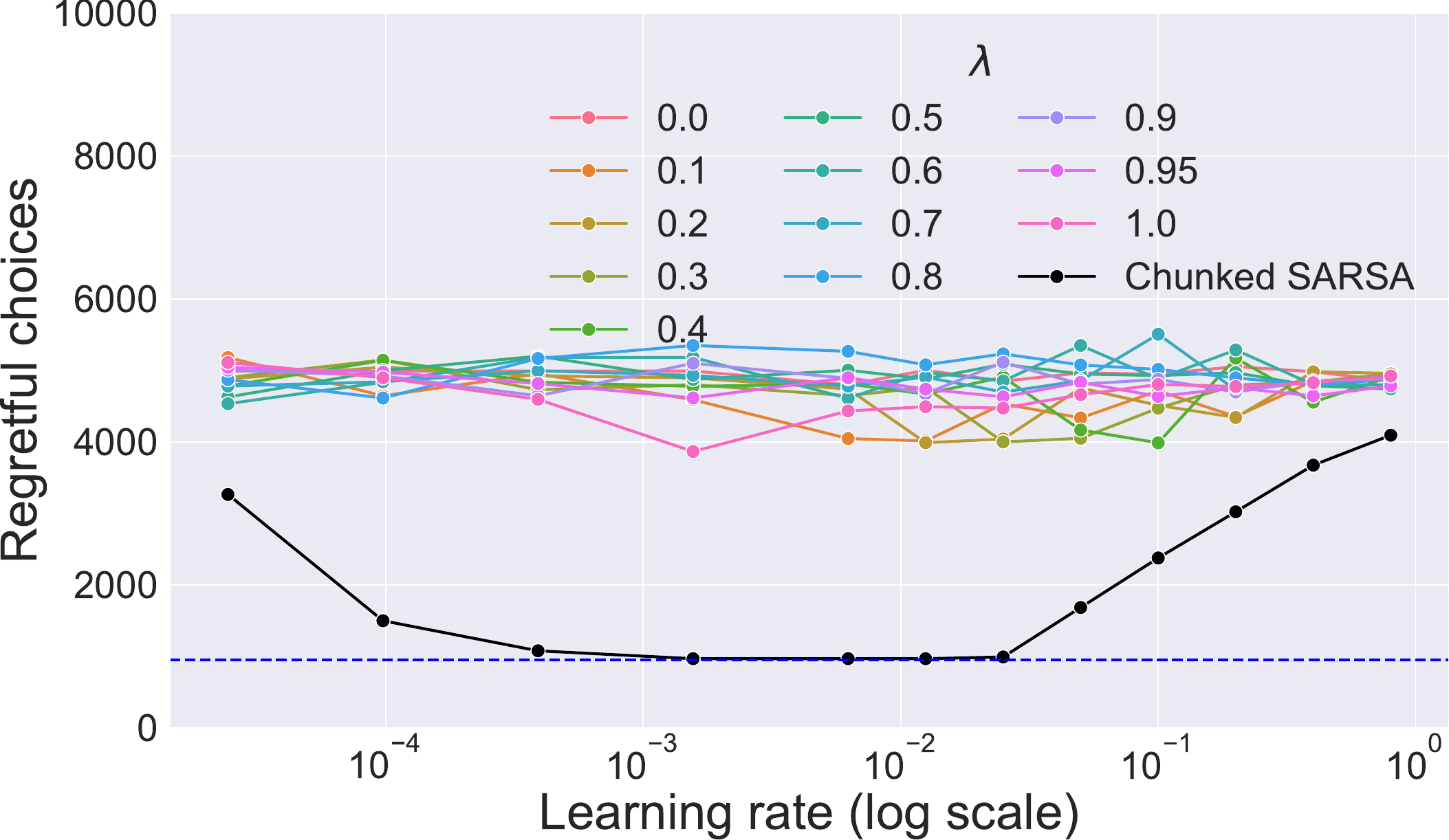}
    \caption{
    Results from the Accumulated-Charge environment.
    The average number of regretful choices over 10 runs for different algorithms with different learning rates for the accumulated-charge experiment.
    The dashed blue line represents the best possible expected regretful choices under the exploration scheme.
    }
\label{fig:results-acc-charge-td}
\end{figure}

In the previous section, we looked at evaluating actions under a uniform random policy. Here, we focus on control, i.e., learning optimal behavior.
We consider a variant of the Choice environment studied by \citet{arjona2019rudder}.
This environment illustrates the benefit of our approach in environments where transitions are mostly deterministic, except at certain points, leading to highly compressible trajectories.

\paragraph{Environment specification}

The agent has two available actions in the initial state ($a_1$ and $a_2$) and only one available action ($a_1$) thereafter.

The state is represented as three components $s^1, s^2, s^3$. 
$s^1$ is a boolean indicating whether the first action was $a_1$.
$s^2$ is a non-negative integer that indicates the total charge accumulated so far.
$s^3$ is the time step.

During the episode, at $k$ values of the time step, the agent stochastically accumulates `charge'.
The $k$ time steps where a charge is accumulated are fixed for the interaction but are randomly chosen for each independent initialization (seed) of the environment.
Further details are presented in Appendix~\ref{app:exps:acc_charge}.
In our selected setting, the episode lasts $H=200$ steps, and the agent encounters $k=10$ accumulation points, where it accumulates charge sampled from a binomial variable with $n=H/k$ and $p=0.5$.

A reward is provided  at the final time step.
The final reward depends on the amount of charge accumulated and on a bonus $b$ that depends on whether the first action was $a_1$ or $a_2$.
Choosing $a_1$ leads to an expected return of $b$, and choosing $a_2$ leads to an expected return of zero. We set $b=0.1$ in our experiment. 
The precise reward function is described in Appendix~\ref{app:exps:acc_charge}.

\paragraph{Evaluation details}

We compare the performance of Chunked SARSA and SARSA($\lambda$) over multiple choices of $\lambda$ and the learning rate.
As the performance measure, we look at the number of regretful choices made by each algorithm.
If the agent's choice at the first time step is not the optimal choice (i.e., $a_1$), it is labeled as regretful.
We learn tabular Q-values.
The agent takes uniform random actions in the first 1000 episodes and then acts $\epsilon$-greedily for 9000 episodes, with $\epsilon=0.1$.

We learn a neural network model to compute $\lambda_{t+1}$ for Chunked SARSA.
The neural network is trained to predict $\Delta s_t = s_{t+1} - s_{t}$ based on $s_t$ and $a_t$.
The neural network is trained on samples from a replay buffer.
Details about the architecture and training are provided in Appendix~\ref{app:exps:acc_charge}.

\paragraph{Results}
 
Figure~\ref{fig:results-acc-charge-td} shows the number of regretful choices made by Chunked SARSA and SARSA($\lambda$) for different values of $\lambda$ and the learning rate.
We see that Chunked SARSA performs best in this environment across several learning rates.
These results are also presented with 95\% bootstrapped confidence intervals in Figure~\ref{fig:acc_charge_with_err_bounds} (Appendix~\ref{app:exps:acc_charge}).
We see that SARSA($\lambda$) exhibits a lot of variance, achieving low regret on some trials and high regret on others.
The reliable success of Chunked SARSA is due to the learned model's accurate prediction of transitions where charge does not accumulate. 
This bridges delays while still bootstrapping to a larger degree upon charge accumulation.

\subsection{Key-to-Door with factored rewards}
\label{exps:k2d}

\begin{table*}[bt]
\setlength{\tabcolsep}{5pt}
    \vspace{-5pt}
    \footnotesize
    \centering
    \caption{
    Number of episodes in which the treasure is not collected (out of 5000). Mean and standard deviation over 10 independent trials.
    }
    \begin{tabular}{lccccccc}
      \toprule
       $\lambda$ & $0$ & $0.1$ & $0.5$ & $0.9$ & $1$ & C-default & C-factored\\
      \midrule
      Episodes &  $1881 \pm 1950$ & $2277 \pm 2068$ & $2273 \pm 2076$ & $2688 \pm 2119$ & $2261 \pm 582$ &  $2285 \pm 2061$& $\mathbf{669 \pm 17}$\\
      \bottomrule
    \end{tabular}
    \vspace{-5pt}
    \label{tab:key-to-door-results}
\end{table*}

A limitation of our Chunked-TD approach is the presence of noisy components that don't alter the expected return.
Since noise is inherently incompressible, our approach would not be able to solve some \emph{hard} credit assignment problems previously considered in the literature~\citep{meulemans2023would, ni2023transformers}.

However, if we can decompose or factor rewards in a certain way, we can apply our chunking idea on a per-component basis to perform well in such environments.

\paragraph{Environment specification}
The environment state comprises of $n_d$ + 4 components.
There are $n_d$ Boolean distractors ($ s_{d_1} \dots s_{d_{n_d}}$), which are sampled independently from a Bernoulli distribution with probability $p=0.5$ at each time step.
The state includes Boolean variables that indicate whether the agent has the key, or is in the door or treasure state ($s_{key}, s_{door}, s_{treasure}$). 
Another component indicates the time step ($s_{time}$).
Further details about the state and reward representation are available in Appendix~\ref{app:exps:k2d}.

The agent has two available actions at every step.
They correspond to ``pick key" and ``unlock door".
Using ``pick key" in the start state turns the ``key" boolean on.
If the agent has the key, using ``unlock door" at the door state (penultimate time step) will lead to the treasure.
Other than this, the actions have no effect on the state.
The environment is episodic with $H=100$ steps and $n_d=4$.

The environment returns a vector of rewards, one for each state component.
If the treasure bit is on, it corresponds to a treasure reward of $0.01$.
Each distractor being on corresponds to a distracting reward of $0.01/(n_d)$.
All other components are associated with zero rewards.
The optimal policy is to pick up the key in the first step and use it to access the treasure in the last step.

\paragraph{Evaluation details}
When rewards are decomposed, we can learn independent tabular Q-values for each reward component, i.e., $n_d + 4$ Q-functions $\hat{Q}_1 (s, a), \hat{Q}_2 (s, a), \dots \hat{Q}_{n_d + 4} (s, a)$. 
Each $Q_i (s, a)$ maps the whole state vector to the Q-value of a particular reward component.
We act according to the global Q-value estimate, which is the sum of component-wise tabular Q-values~\citep{russell2003q}, $\hat{Q} (s, a) = \sum_i^{n_d + 4} Q_i (s, a)$.
See Appendix~\ref{app:algorithm_variants} for additional details.
We compare three approaches to learning Q-values.

The first corresponds to Expected-SARSA($\lambda$), where the value of $\lambda$ is a fixed scalar and remains the same for each $Q_i (s, a)$.
The second is the chunked counterpart of Expected-SARSA (Algorithm~\ref{alg:chunked_td_q_v}), but uses the same value of $\lambda_{t+1}$ for each $Q_i (s, a)$, where $\lambda_{t+1} = \hat{P}^{\pi} (x_{t+2} | x_{t+1})$ as is usual.
We refer to this variant as \emph{C-default}.

The third is Chunked Expected-SARSA, which uses different values of $\lambda^i_{t+1}$ for each $Q_i (s, a)$. 
We refer to this algorithm as \emph{C-factored}.
Specifically, the bootstrapping factor of each component only depends on the predictability of that component, $\lambda^i_{t+1} = \hat{P}^{\pi} (x^i_{t+2} | x_{t+1})$.
The exact algorithm is presented in Appendix~\ref{app:algorithm_variants}.

Actions are selected $\epsilon$-greedily from the global Q-value. $\epsilon$ is annealed from 1 to 0.1 over the first 500 episodes.
The agent interacts with the environment for 5000 episodes. We compare algorithms based on the number of episodes in which the treasure was not collected.
We consider multiple learning rates for each approach (see Appendix~\ref{app:exps:k2d}).

\paragraph{Results}

Table~\ref{tab:key-to-door-results} presents results for the best choice of learning rate for each approach.
Results with all considered learning rates are presented in Figure~\ref{fig:k2d-all-lrs}.
Expected-SARSA($\lambda$), with $\lambda \in \{ 0, 0.1, 0.5, 0.9, 1\}$, does not reliably learn to collect the treasure.
Similarly, the default version of Chunked Expected-SARSA (C-default) performs poorly.

Since rewards are suitably decomposed, we could use the factored version of Chunked Expected-SARSA (C-factored).
C-factored bootstraps separately based on the predictability of each component of the state.
For e.g., notice that the key, door, treasure, and time steps are perfectly predictable from the previous state.
C-factored behaves like an MC update with respect to the value of only these factors. 
This allows the agent to better estimate Q-values quickly and successfully collect the treasure in most episodes.

\section{Related Work}
The algorithms introduced in this paper are closely related to the ones proposed by \citet{sutton1994step}.
Their algorithms are motivated from the perspective of eliminating the bias of temporal difference algorithms, which is different from the compression-based motivation~\citep{Schmidhuber:92ncchunker} we present here.
Similar to our approach, their Corrected-TD and TD($n/n$) algorithms use estimated forward transition probabilities from tabular models to decay traces.
They also use state visitation counts to adapt the learning rate.
As we show in Appendix~\ref{app:sns94}, their algorithm amounts to using a $\lambda$-return based on a `backward' dynamics model's predictions.
In contrast, our $\lambda$-return is defined through predictions of a forward model. 
Forward models can be easier to learn as they are not policy-dependent and can use replay buffers for training~\citep{chelu2020forethought}.

Unlike our model-based approach to select $\lambda$, other works have studied model-free RL algorithms with adaptive or state-dependent values of $\lambda$.
\citet{white2016greedy} discuss several factors that could guide the choice of $\lambda$ in TD algorithms: bias-variance control, confidence in value estimates, etc.
They design an objective that optimizes the immediate bias-variance trade-off to generate targets as a mixture of TD(0) and MC.
\citet{riquelme2019adaptive} propose an approach to policy evaluation that switches between TD(0) and MC based on confidence estimates of the value at that state.
\citet{watkins1989learning} proposes the Q($\lambda$) algorithm, which cuts the trace whenever an exploratory action is taken.
\citet{xu2018meta} use meta-gradients to adjust $\lambda$.

A differentiable world model can be used to directly obtain gradients for policy improvement~\citep{Werbos:87, Munro:87, Schmidhuber:90diffenglish}.
In deterministic environments with continuous actions, we can update a policy by coupling it with a differentiable world model.
This idea has been extended to handle stochastic environments/policies and to use actual trajectories~\citep{heess2015learning}.
However, these approaches are not easily applicable to environments with discrete actions and are quite sensitive to inaccuracies in the learned model~\citep{DBLP:conf/iclr/HafnerL0B21, meulemans2023would}.
Our chunking-based approach presents a new way of using world models with actual trajectories for credit assignment in environments with discrete actions.

Learned models can also be used for imagination-based planning~\citep{Schmidhuber:90diffenglish,  Sutton:90dyna, ha2018world,
DBLP:conf/iclr/HafnerL0B21}.
Experience replay can be seen as a non-parametric model and be used to drive a similar effect~\citep{van2019use}.
Using backward-model-based imagination can also speed up credit assignment by assigning credit to all state-actions that could have led to the current outcome~\citep{moore1993prioritized, goyal2018recall, pitis2018source, van2021expected}.
The benefits of using imagined experience are orthogonal to our proposed approach which focuses on assigning credit based on real experience.

Alternatives to standard forward dynamics models can be used for assigning credit with real experience.
The hindsight credit assignment (HCA) family of approaches learns a temporally extended inverse dynamics model to ascertain the influence of actions~\citep{harutyunyan2019hindsight, alipov2021towards, meulemans2023would}.
Similar to backward models, temporally extended inverse dynamics models are policy-dependent and can be challenging to learn.

Several works identify the need to move beyond conventional TD($\lambda$) to tackle challenging temporal credit assignment problems; see the recent survey by \citet{pignatelli2023survey}. Finally, \citet{arumugam2021information} present an information-theoretic formulation of the credit assignment problem that may present interesting connections to our compression-based approach.

\section{Limitations}

A crucial limitation of our approach is that learning a model of the world can be challenging and expensive.
Furthermore, our approach requires a `generative' model that can estimate the probability of the next state, given the current state and action.
Nevertheless, our approach could be particularly useful when pre-trained generative models are already available or are being learned for imagination-based planning.

In this paper, we focus on tabular value functions to avoid conflating the generalization effects of function approximation with credit assignment while developing an initial understanding of Chunked-TD.
Further work is required to extend Chunked-TD to continuous states and actions where probability mass functions are unavailable.
One possibility is to chunk with discretized tokens~\citep{janner2021offline} or with discrete latent representations (e.g., \citealp{van2017neural}).
Another option is to apply a transformation on the error in the model’s prediction to obtain $\lambda_t$.
Systematically extending our approach to such settings requires careful consideration of additional design choices (and associated hyperparameters) and is deferred to future work.

Our approach relies on a suitable reward decomposition to solve hard credit assignment problems like the key-to-door (Section~\ref{exps:k2d}).
Learning such decompositions is challenging in general, but some works aim to tackle similar problems through the use of successor features~\citep{barreto2017successor}, reward features~\citep{meulemans2023would} and general value functions~\citep{van2017hybrid}. 

\section{Conclusion}

In this paper, we introduce Chunked Temporal Difference, an approach that uses predictions from a learned model to construct adaptive $\lambda$-returns for value learning in on-policy settings.
Our work presents a new perspective on how learned world models can be used for credit assignment.
Building upon the idea of history compression, we use model predictions to reduce trajectory descriptions, thus shortening credit assignment paths.

We propose algorithms that can be implemented online and show that they solve some problems much faster than conventional TD($\lambda$).
Our approach introduces a promising way to use a learned model that is less vulnerable to model inaccuracies.
Since our approach only uses the learned model to chunk trajectories, it can also be useful when models are only accurate in some regions of the state space.

The clearest direction for future work is to extend this idea to higher dimensional state spaces and continuous actions.
Other exciting next steps are to generalize our approach to the off-policy setting and to further explore different methods of chunking trajectories, e.g., with policy-conditioned backward models or even in model-free ways.

\section*{Acknowledgements}

We would like to thank Kazuki Irie for helpful discussions.
This research was supported by an ERC Advanced Grant (742870) and a Swiss National Science Foundation grant (200021\_192356). Computational resources for this work were provided by the Swiss National Supercomputing Centre (CSCS project s1205).

\section*{Impact Statement}

This paper presents work whose goal is to advance the field of Machine Learning. There are many potential societal consequences of our work, none of which we feel must be specifically highlighted here.

\bibliography{references}
\bibliographystyle{icml2024}

\newpage
\appendix
\onecolumn

\section{Experiments}

\subsection{Chain-and-Split}
\label{app:exps:cns}

\subsubsection{Environment details}

The environment is depicted in Figure~\ref{fig:simple_env}.
The agent has $n$ actions $\{ a_1, a_2, \dots a_n\}$ available  at the start/root state.
After the first step, there is only one action available in every state.
Taking $a_1$ at the starting state leads to a linear chain of length $H$, which ultimately provides a deterministic reward of $0.01$.
This is the optimal choice.
All remaining actions lead to a common `parent' state on the left ($s_L$), which has a value of zero but is followed immediately by a stochastic transition that leads to a range of different rewards.
The state $s_L$ branches to one of $w$ states with equal probability. 
Each of the $w$ states provides a deterministic reward between -1 and 1, with the average over all states being equal to zero.
We set $H=20$, $n=10$, and $w=101$ for our experiment.
The discount factor $\gamma=1$.

\subsubsection{Implementation details}
We learn tabular Q-value functions for SARSA($\lambda$) and Chunked SARSA.
The tabular Q-values are initialized at zero.

\paragraph{Chunked SARSA}
We use Algorithm~\ref{alg:chunked_td_q_q}.
To model $\hat{P}^{\pi} (x_{t+1}, a_{t+1} | x_{t}, a_{t})$ in this case, we simply maintain empirical counts of each possible transition $n(x, a, x')$ and state-action pair $n(x, a)$ and let $\hat{P} (x_{t+1} | x_t, a_t) = \frac{n (x_t, a_t, x_{t+1})}{n (x_t, a_t)}$.
We decompose $\hat{P}^{\pi} (x_{t+1}, a_{t+1} | x_{t}, a_{t}) = \hat{P} (x_{t+1} | x_t, a_t) \pi (a_{t+1} | s_{t+1})$. 
Recall that $x_t$ is the same as $s_t$ in our experiments.

The tabular model is updated at every transition, and counts are incremented before they are used for computing $\hat{P}^{\pi} (x_{t+1}, a_{t+1} | x_{t}, a_{t})$.
This ensures that the counts are never zero for any observed transition.

\paragraph{Hyperparameter selection}

We consider multiple values of the learning rate $\alpha$ for SARSA($\lambda$) with $\lambda \in \{ 0, 1.0 \}$ and for Chunked SARSA.
We consider $\alpha \in 0.1 \times \{2^{-2}, 2^{-3}, 2^{-4}, 2^{-5}, 2^{-6}, 2^{-7}, 2^{-8}, 2^{-9}, 2^{-10}
      2^{-11}, 2^{-12}, 2^{-13} \}$.

The selected choices are provided in Table~\ref{tab:hgrid_cns}.
The selected learning rate for SARSA(0) and Chunked SARSA is the one with the least square error with the true value of $\Delta Q$ after 100000 episodes.

Since SARSA(1) is quite poor initially, 
we selected the learning rate with the least square error with $\Delta Q$ after 1000000 episodes instead of 100000 episodes.
SARSA(1) fares poorly as it tends to have negative values of $\Delta Q$ due to overestimating the value of some action. 

\begin{table*}[ht]
    \small
    \centering
    \caption{Selecting the learning rate for Chain-and-Split experiments.}
    \label{tab:hgrid_cns}
    \begin{tabular}{l  c}
      \toprule
      Algorithm &  Selected $\alpha$ \\
      \midrule
      SARSA(0)  & $0.1 \times 2^{-5}$ \\
      SARSA(1)  & $0.1 \times 2^{-11}$ \\
      Chunked SARSA & $0.1 \times 2^{-8}$  \\
      \bottomrule
    \end{tabular}
\end{table*}

Figure~\ref{fig:results_q_gap_simple_env} in the main paper shows results for $\lambda \in \{ 0, 1.0\}$ and for Chunked SARSA with the selected learning rate.

\subsubsection{Additional results}

Here we show that SARSA(0) and SARSA(1)/MC eventually achieve positive values of $\Delta Q$ with a longer evaluation.

We use the same hyperparameters as those selected in Table~\ref{tab:hgrid_cns} and evaluate for 1000000 episodes instead of the 100000 episode evaluation of the main paper.
Note that the chosen learning rates may not be the best selection for SARSA(0) and Chunked SARSA with 1000000 episodes.

These results are presented in Figure~\ref{fig:results_q_gap_simple_env_long}.

\begin{figure}
    \centering
    \includegraphics[width=0.7\textwidth]{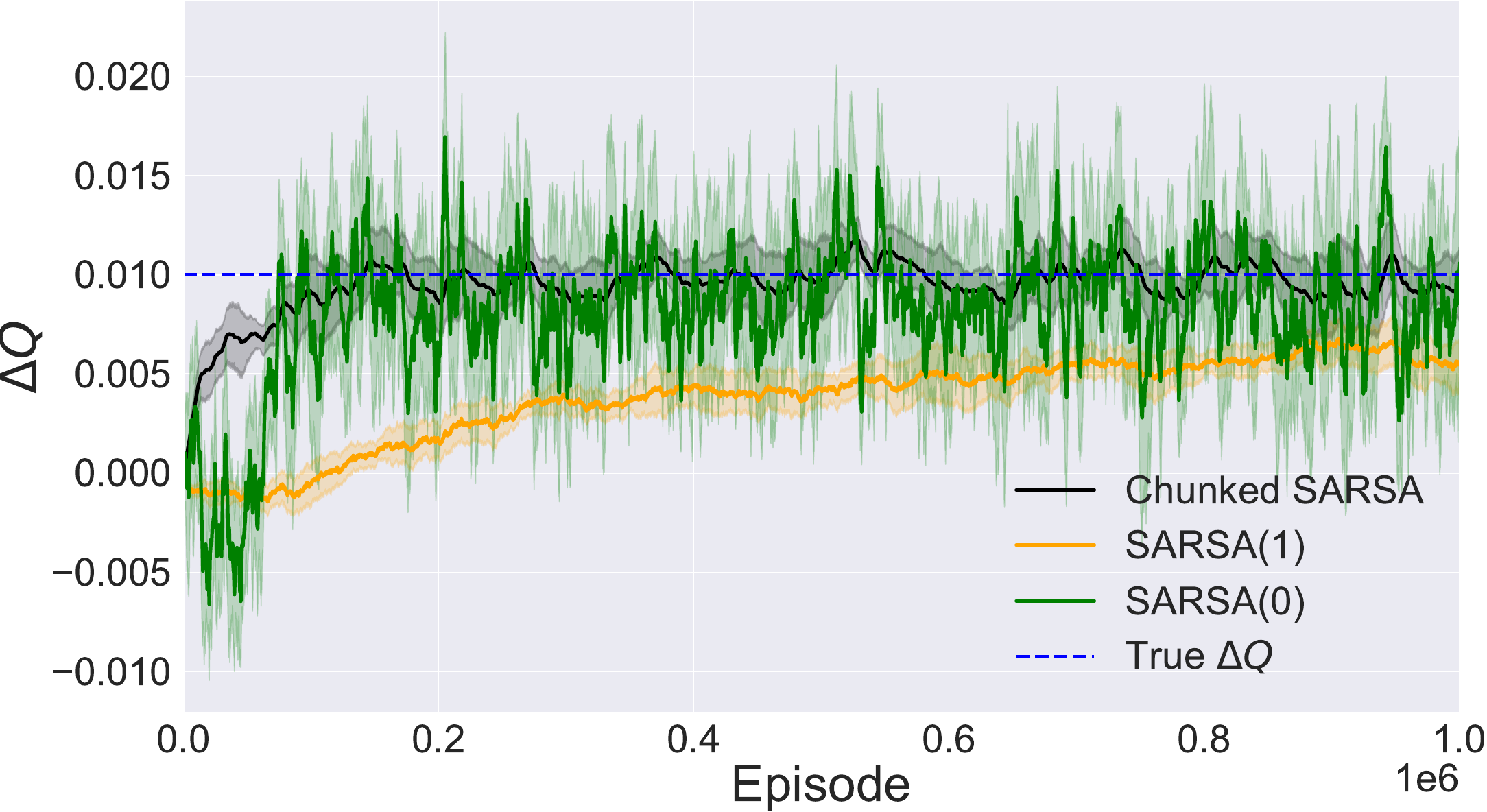}
    \caption{
    The action-value gap ($\Delta Q$) between optimal action $a_1$ and the maximum action value among the remaining actions.
    The true value of $\Delta Q = 0.01$. 
    Shading indicates 95\%
    bootstrapped confidence intervals over 10 independent trials.
    These results present evaluation over 10 times the number of episodes as those presented in Figure~\ref{fig:results_q_gap_simple_env}
    }
    \label{fig:results_q_gap_simple_env_long}
\end{figure}

\subsection{Accumulated-Charge}
\label{app:exps:acc_charge}

\subsubsection{Environment details}

The agent has two available actions in the initial state ($a_1$ and $a_2$) and only one available action ($a_1$) thereafter.
The environment is episodic with $H$ time steps. 
The discount factor $\gamma=1$.

The state has three components $s = (s^1, s^2, s^3)$. 
$s^1$ is a boolean indicating whether the first action was $a_1$,
$s^2$ is a non-negative integer that indicates the total charge accumulated so far,
and $s^3$ is the time step.

During the episode, at $k$ values of the time step, the agent stochastically accumulates `charge'.
The $k$ time steps where a stochastic charge is accumulated are fixed for the interaction (all episodes) but are randomly chosen for each independent initialization (seed) of the environment.
In our selected setting, the episode lasts $200$ steps, and the agent encounters $k=10$ accumulation points, where it accumulates charge sampled from a binomial variable with $n=H/k=20$ and $p=0.5$.

There are no rewards apart from those provided at the final time step.
The final reward depends on the amount of charge accumulated and on a bonus $b$ that depends on whether the first action was $a_1$ or $a_2$.
Choosing $a_1$ leads to an expected return of $b$, and choosing $a_2$ leads to an expected return of zero. We set $b=0.1$ in our experiment.

\paragraph{Reward definition}
The final reward (at $H$) is a sum of three terms $R_H = R^b + R^c + R^d$. 

Let $c_0 = 0.5$ if $s_H^1 = 1$ (i.e. the agent took $a_1$ at $s_0$).
Otherwise, if $s_H^1 = 0$ (i.e. the agent took $a_2$ at $s_0$), $c_0 = -0.5$.

We can now define the individual terms.
The bonus term is provided if the agent had taken $a_1$ in $s_0$, $R^b = s_H^1 b$.
The next term $R^c$ depends on the amount of charge accumulated on the trajectory, $R^c = s_H^2 c_0$.
The final component is $R^d = - c_0 p H$. Note that $\mathbb{E} [R^c  + R^d] = 0$, since $\mathbb{E} [ s_H^2 ] = pnk = pH$.

\paragraph{Randomization of charging time steps}

Evaluation in this environment may be considered as an evaluation over a distribution of MDPs.
The choice of the environment seed fixes the location of the time steps where the charge accumulates. These steps remain fixed for all episodes in the environment instance.

\subsubsection{Implementation details}
We learn tabular Q-value functions for SARSA($\lambda$) and Chunked SARSA.
The tabular Q-values are initialized at zero.

\paragraph{Chunked SARSA}
We use Algorithm~\ref{alg:chunked_td_q_q}.
We learn a neural network model to predict $\hat{P}^{\pi} (x_{t+1}, a_{t+1} | x_{t}, a_{t})$.
Recall that $x_t = s_t$ in our experiments.

We can decompose $\hat{P}^{\pi} (x_{t+1}, a_{t+1} | x_{t}, a_{t}) = \hat{P} (x_{t+1} | x_t, a_t) \pi (a_{t+1} | s_{t+1})$.

Since the state space is quite large, with 200 possible values of charge and time step, we learn a neural network model to  predict $\Delta s_{t+1} = s_{t+1} - s_t$ as a categorical distribution.
Using state-delta has the benefit of reducing the space of outputs under the model and is similar to models learned in video prediction.
The probability of the whole next state (or state-delta) is modeled independently, i.e., taken as the product of the probability of each individual component (each component is a categorical variable in $n=H/k$, the maximum change possible in a state component in this environment).
We use the probability of the actual delta and use $\hat{P} (\Delta s_{t+1} | s_t, a_t)$ as the model's predicted probability of the transition $\lambda_{t+1} = \hat{P} (\Delta s_{t+2} | s_{t+1}, a_{t+1}) \pi (a_{t+1} | s_{t+1})$.

Our neural network has 3 hidden layers with 256 units each and hyperbolic tangent activation functions.
The neural network is trained on a batch sampled from the replay buffer every $k_{model}$ steps in the environment.

\paragraph{Hyperparameter selection}
The hyperparameters relating to the neural network model were selected as  reasonable values based on preliminary experiments.
These are presented in Table~\ref{tab:nn_hparams_acc_charge}.
We use the Adam optimizer for training the neural network~\citep{kingma2015adam}.
All other hyperparameters match PyTorch defaults.

\begin{table*}[ht]
    \small
    \centering
    \caption{NN model hyperparameters for accumulated charge experiment.}
    \label{tab:nn_hparams_acc_charge}
    \begin{tabular}{l c c}
      \toprule
      Hyperparameter & Value \\
      \midrule
      Model trained every $k_{model}$ steps & 4 \\
      Batch size & 128 \\
      Model learning rate ($\eta$) &  0.0001 \\
      Weight decay for optimizer &  1e-6 \\
      Replay buffer size & 100000 \\
      \bottomrule
    \end{tabular}
\end{table*}

We present results for SARSA($\lambda$) with $\lambda \in \{0, 0.1, 0.2, 0.3, 0.4, 0.5, 0.6, 0.7, 0.8, 0.9, 0.95, 1\}$ and Chunked SARSA.
We consider the following learning rates ($\alpha$)for all approaches and report results for all settings, $\alpha \in 0.1 \times\{ 2^{3}, 2^{2}, 2^{1}, 2^{0}, 2^{-1}, 2^{-2}, 2^{-3}, 2^{-4}, 2^{-6},
2^{-8}, 2^{-10}, 2^{-12} \}$.

\subsubsection{Additional results}

In the main text, we presented the average number of regretful choices for each approach and learning rate.
Here, we present the same results with 95\% bootstrapped confidence intervals (see Figure~\ref{fig:acc_charge_with_err_bounds}).
The confidence intervals are quite wide for TD($\lambda$) approaches as they prefer $a_1$ or $a_2$ at random (across seeds) after the initial exploratory phase.
They fail to reliably prefer $a_1$, resulting in a high average number of regretful choices with high variance.

\begin{figure}
    \centering
    \includegraphics[width=0.7\textwidth]{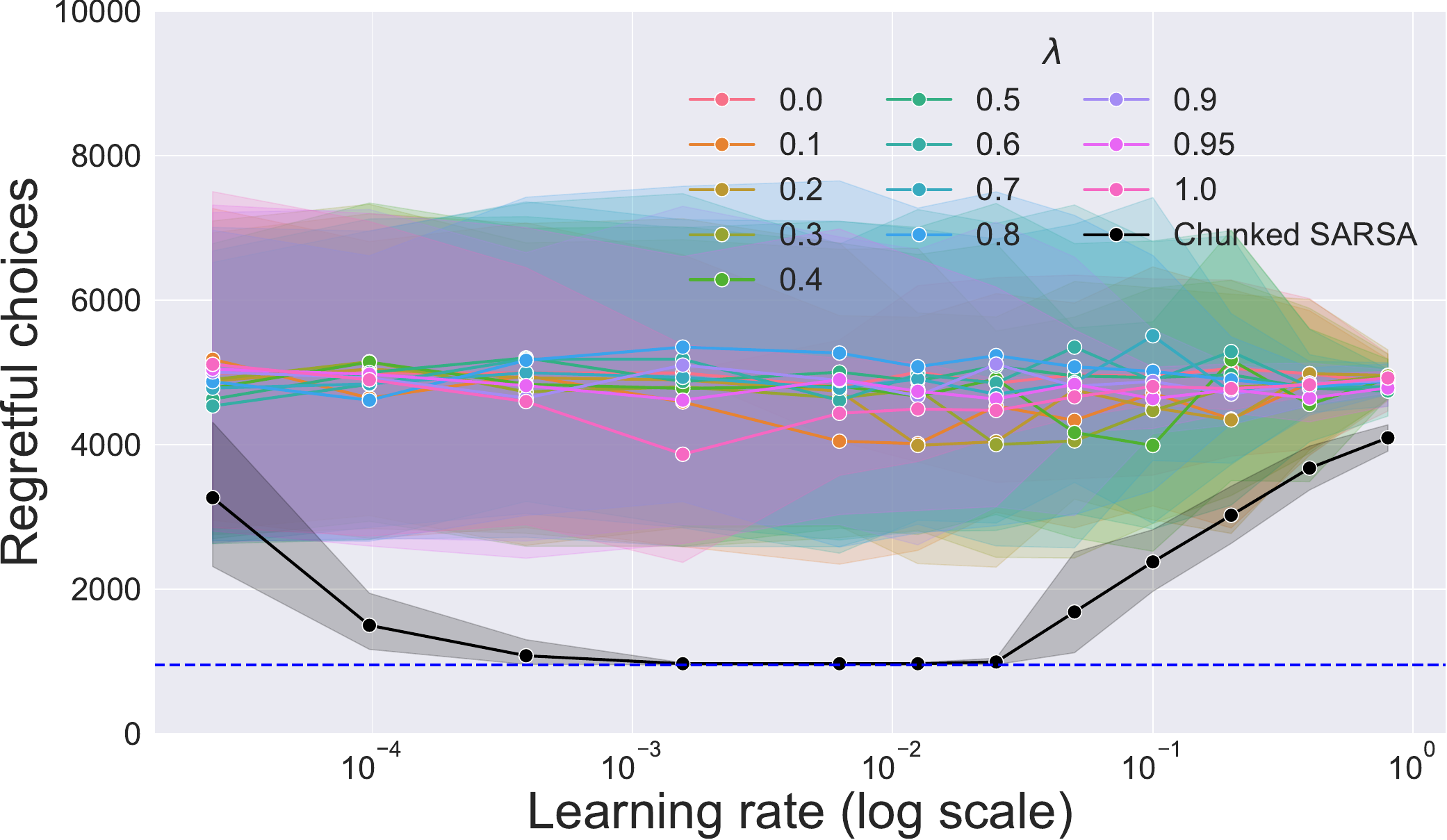}
    \caption{
    Results from the accumulated-charge environment.
    The average number of regretful choices over 10 runs for different algorithms with different learning rates for the accumulated-charge experiment.
    The shading represents 95\% bootstrapped confidence intervals.
    The dashed blue line represents the best possible number of expected regretful choices under the exploration scheme.
    }
    \label{fig:acc_charge_with_err_bounds}
\end{figure}

\subsection{Key-to-Door}
\label{app:exps:k2d}

\subsubsection{Environment details}
The state comprises of $n_d$ + 4 components, $s = (s_{key}, s_{door}, s_{d_1} \dots s_{d_{n_d}} s_{treasure}, s_{time} ).$ 
There are $n_d$ Boolean distractors ($ s_{d_1} \dots s_{d_{n_d}}$), which are sampled independently from a Bernoulli distribution with probability $p=0.5$.
One boolean indicates whether the agent has the key ($s_{key}$). 
Two other components are Booleans that indicate whether the agent is in front of a door ($s_{door}$), and if the current state has the treasure ($s_{treasure}$). 
The last component indicates the time step ($s_{time}$).
Distractors are absent (with a zero value) in the start state and states with treasure or door bits on.

The agent has two available actions at every step.
They correspond to ``pick key" and ``unlock door".
Using ``pick key" in the start state turns the ``key" boolean on.
If the agent has the key, using ``unlock door" at the penultimate state will lead to the treasure.
Other than this the actions have no effect on the state.
The environment is episodic with $H$ steps.
The discount factor $\gamma=1$.

The environment returns a vector of rewards, one for each state component.
If the treasure bit is on, it corresponds to a treasure reward of $0.01$.
Each distractor being on corresponds to a distracting reward of $0.01/(n_d)$.
All other components, the key, door, and time step, are associated with zero rewards. 
Concretely, the reward vector is a deterministic linear function of the next state, $\vec{r}_{t+1} = s_{t+1} \odot (0, 0, \frac{0.01}{n_d} \dots \frac{0.01}{n_d}, 0.01, 0)$, where $\odot$ represents the element-wise product.
The overall scalar reward can be obtained by summing over the reward vector.

In our experiment, we set $H=100$ and $n_d=4$.

\subsubsection{Implementation details}

We use variants of Chunked Expected-SARSA in these experiments (Algorithm~\ref{alg:chunked_td_q_v_f}).
Further, since the reward provided by the environment is a vector of reward components, we learn a separate Q-value for each component.
See Appendix~\ref{chunk-exp-sarsa-f} for details.

We learn tabular Q-value functions for each reward component. All values are initialized as zero.

\paragraph{C-factored}
The C-factored algorithm refers to Algorithm~\ref{alg:chunked_td_q_v_f}.
We decompose $\hat{P}^{\pi} (x^i_{t+1},| x_{t}) = \sum_a \hat{P} (x^i_{t+1} | x_t, a) \pi (a | s_{t}).$
Recall that $x_t = s_t$ in our experiments.

\paragraph{Hyperparameter selection}

We learn a neural network model that predicts the next state components given the current state and action ($\hat{P} (s_{t+1} | s_t, a)$).
All state components apart from the time step are Booleans, which we model as independent Bernoulli variables.
We predict the time step as a categorical variable with $H$ categories.

Our neural network has 2 hidden layers with 128 units each and hyperbolic tangent activation functions.
The neural network is trained on a batch sampled from the replay buffer every $k_{model}$ steps in the environment.

The hyperparameters relating to the neural network model were selected as  reasonable values based on preliminary experiments.
These are presented in Table~\ref{tab:nn_hparams_k2d}.
We use the Adam optimizer for training the neural network~\citep{kingma2015adam}.
All other hyperparameters match PyTorch defaults.

\begin{table*}[ht]
    \small
    \centering
    \caption{NN model hyperparameters for key-to-door}
    \label{tab:nn_hparams_k2d}
    \begin{tabular}{l c c}
      \toprule
      Hyperparameter & Value \\
      \midrule
      Model trained every $k_{model}$ steps & 1 \\
      Batch size & 64 \\
      Model learning rate ($\eta$) &  0.0002 \\
      Replay buffer size & 10000 \\
      \bottomrule
    \end{tabular}
\end{table*}
 
We conducted a sweep over learning rates for all approaches.
We show results for the best learning rate in terms of the minimum number of episodes in which the treasure was not collected (averaged over 10 seeds).
The considered candidates and the selected values are provided in Table~\ref{tab:hgrid_k2d}.

\begin{table*}[ht]
    \small
    \centering
    \caption{Selecting the learning rate for Key-to-Door experiments.}
    \label{tab:hgrid_k2d}
    \begin{tabular}{l  cc}
      \toprule
      Algorithm &  Candidates & Selected $\alpha$ \\
      \midrule
      Expected-SARSA(0) & $ 0.1 \times\{2^{2}, 2^{1}, 2^{0}, 2^{-1}, 2^{-2}, 2^{-3}, 
      2^{-4}, 2^{-5},
      2^{-6}, 2^{-7},
      2^{-8}, 2^{-9}
      \}$ & $0.1 \times 2^{1}$ \\
      Expected-SARSA(0.1) & 
      $ 0.1 \times\{2^{2}, 2^{1}, 2^{0}, 2^{-1}, 2^{-2}, 2^{-3}, 
      2^{-4}, 2^{-5},
      2^{-6}, 2^{-7},
      2^{-8}, 2^{-9}
      \}$& $0.1 \times 2^{2}$ \\
      Expected-SARSA(0.5) &       $ 0.1 \times\{2^{2}, 2^{1}, 2^{0}, 2^{-1}, 2^{-2}, 2^{-3}, 
      2^{-4}, 2^{-5},
      2^{-6}, 2^{-7},
      2^{-8}, 2^{-9}
      \}$ & $0.1 \times 2^{0}$\\
      Expected-SARSA(0.9) &  $ 0.1 \times\{2^{2}, 2^{1}, 2^{0}, 2^{-1}, 2^{-2}, 2^{-3}, 
      2^{-4}, 2^{-5},
      2^{-6}, 2^{-7},
      2^{-8}, 2^{-9}
      \}$ & $0.1 \times 2^{0}$ \\
      Expected-SARSA(1) & $ 0.1 \times\{2^{2}, 2^{1}, 2^{0}, 2^{-1}, 2^{-2}, 2^{-3}, 
      2^{-4}, 2^{-5},
      2^{-6}, 2^{-7},
      2^{-8}, 2^{-9}
      \}$ & $0.1 \times 2^{-3}$ \\
      C-default & $ 0.1 \times\{2^{2}, 2^{1}, 2^{0}, 2^{-1}, 2^{-2}, 2^{-3}, 
      2^{-4}, 2^{-5},
      2^{-6}, 2^{-7},
      2^{-8}, 2^{-9}
      \}$ & $0.1 \times 2^{0}$  \\
      C-factored & $ 0.1 \times\{2^{2}, 2^{1}, 2^{0}, 2^{-1}, 2^{-2}, 2^{-3} \}$ & $0.1 \times 2^{-1}$  \\
      \bottomrule
    \end{tabular}
\end{table*}

\subsubsection{Additional Results}

We present results for all considered learning rates (from Table~\ref{tab:hgrid_k2d}) in Figure~\ref{fig:k2d-all-lrs}.
We see that C-factored successfully collects the treasure in most episodes across multiple values of the learning rate.

\begin{figure}
    \centering\includegraphics[width=0.7\textwidth]{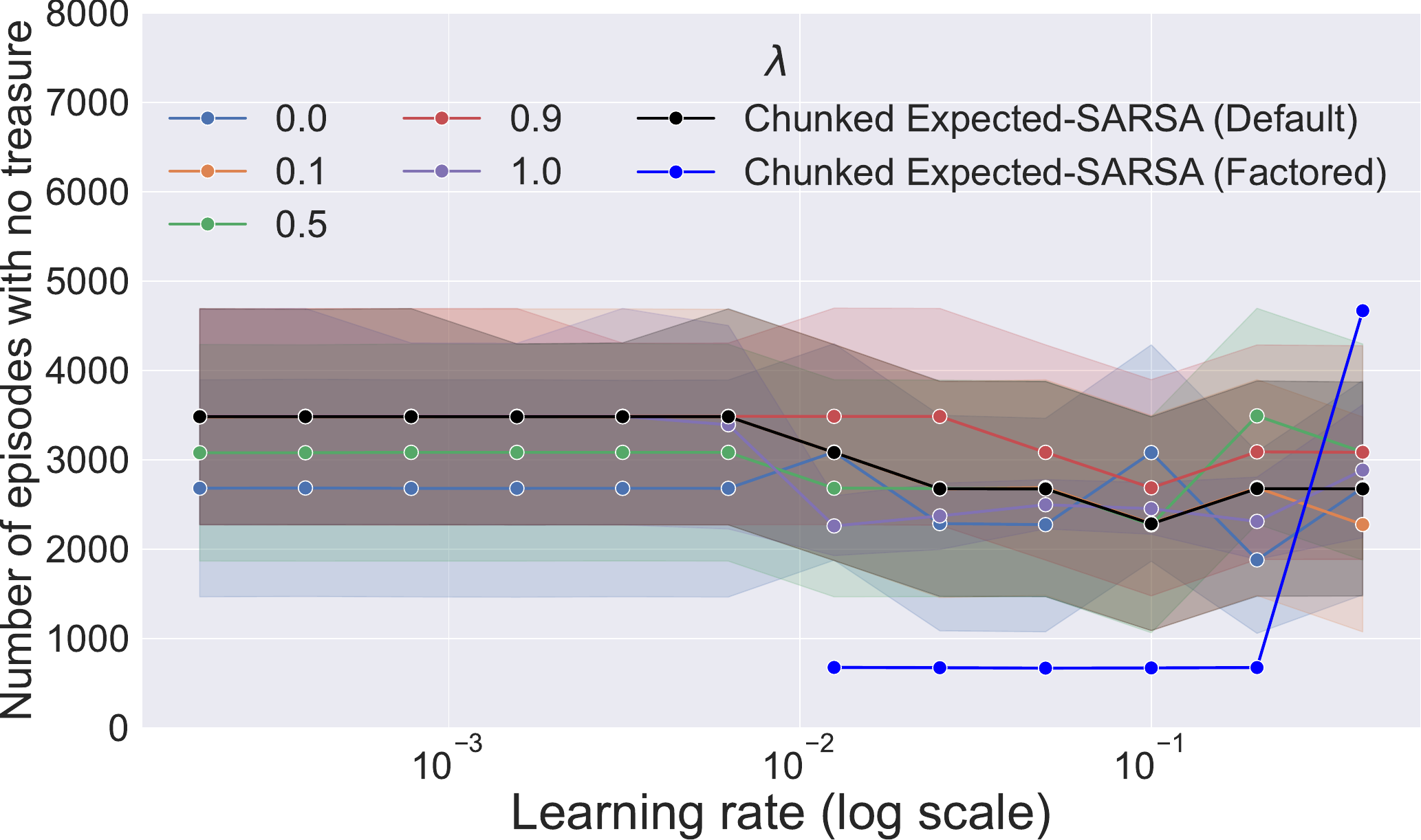}
    \caption{
    The solid line represents the average number of episodes in which the treasure was not collected.
    Shading represents 95\% bootstrapped confidence intervals.
    }
    \label{fig:k2d-all-lrs}
\end{figure}

\section{Variants of our algorithm}
\label{app:algorithm_variants}

In this section, we present and discuss variants of our chunking-based algorithms.

\subsection{n-step Chunked SARSA}

Section~\ref{sec:method:subsec:nstep} presents a chunked $n$-step method for bootstrapping from state value functions.
Here, we also present an $n$-step SARSA-like chunking procedure.

In this variant, which applies to bootstrapping from Q-values, we use $\hat{P}^{\pi}(X_{t+1} A_{t+1}| X_t, A_t)$, the probability of the next percept-action under the current policy. The following factorization can be used in MDPs with Markov policies,
\begin{equation}
\hat{P}^{\pi}(X_{t+1}, A_{t+1}| X_t, A_t) = \hat{P}(X_{t+1} | X_t, A_t) \pi( A_{t+1} | X_{t+1}), 
\end{equation}
where $\hat{P}(X_{t+1} | X_t, A_t)$ is the estimated transition probability.

Consider the following chunking strategy that applies to learning $\hat{Q}(s, a)$ in episodic MDPs: (1) We always keep the first state-action $(s_0, a_0)$ (2) For $t \geq 1$ we drop state-action $(s_t, a_t)$ with probability $\hat{P}^{\pi}(x_{t+1}, a_{t+1} | x_t, a_t)$ (3) Whenever we drop a percept-action, we sum the intermediate rewards such that the compressed and original episode have equivalent returns (4) We always keep the terminal state $s_T$.

The compressed episode of state-action pairs (and the final state) is used to construct targets for the state-action pairs encountered in the entire episode.

\subsection{Chunked SARSA and Chunked Expected-SARSA}

In this section, we present online algorithms for Chunked SARSA and Chunked Expected-SARSA.

\paragraph{Chunked SARSA}
Chunked SARSA is presented in Algorithm~\ref{alg:chunked_td_q_q}.
The forward view of the algorithm would correspond to using $\lambda_{t+1} = \hat{P}^{\pi} (x_{t+2}, a_{t+2} | x_{t+1}, a_{t+1})$ in Equation~\ref{eqn:variable_lambda_return}.
We used the Chunked SARSA algorithm for our experiments in Sections~\ref{exps:cns} and \ref{exps:acc_charge}.

\paragraph{Chunked Expected-SARSA}
Chunked Expected-SARSA is presented in Algorithm~\ref{alg:chunked_td_q_v}.
The forward view of the algorithm would correspond to using $\lambda_{t+1} = \hat{P}^{\pi} (x_{t+2} | x_{t+1})$.
As is the case with Expected-SARSA, the value of the next state is computed as a weighted average of Q-values under the policy.
We use Chunked Expected-SARSA for our experiment in Section~\ref{exps:k2d}.

An advantage of the Chunked Expected-SARSA over Chunked SARSA is in situations where multiple actions lead to the same next state.
Notice that Chunked SARSA would cut the trace when an action of low probability is taken, even if the next state is the same for all actions ($\lambda_{t} = \hat{P}(X_{t+1} | X_t, A_t) \pi( A_{t+1} | X_{t+1})$).
This would not happen with Chunked Expected-SARSA, as it averages over the consequence of all actions, i.e., $\lambda_{t} = \sum_a \hat{P}(X_{t+1} | X_t, a) \pi( a | X_{t})$.

\subsection{Component-wise chunked-TD}
\label{chunk-exp-sarsa-f}

The algorithm presented here follows from a combination of Chunked Expected-SARSA (Algorithm~\ref{alg:chunked_td_q_v}) and Q-decomposition~\citep{russell2003q}.

We assume the state has $d$ components, each associated with its own reward.
Formally, at each time step we receive $\vec{R}_{t+1} = (R^1_{t+1}, R^2_{t+1}, \dots R^d_{t+1})$ and $S_{t+1} = (S^1_{t+1}, S^2_{t+1}, \dots S^d_{t+1})$.
Our algorithm relies on the fact that each reward component is associated with a single state component, i.e., $R^i_{t+1} = f (S^i_{t+1})$.

The typical RL reward is the sum of all the component rewards $R_{t+1} = \sum_{i=1}^d R^i_{t+1}$, whose expected future sum we would like to maximize.

We learn $d$ Q-value functions $\hat{Q}^1 (s,a), \hat{Q}^2 (s,a) \dots \hat{Q}^d (s,a)$.
Each is learned by Chunked Expected-SARSA
with its own trace ($e^i (s, a) $), which is based on the predictability of each individual state/percept component $ \hat{P}^{\pi} (X^i_{t+1} | X_t)$.
Since the reward is the sum of individual components, the Q-value also decomposes in the same way.
Actions are taken according to the global/overall Q-value $\hat{Q} (s,a) = \sum_{i=1}^d \hat{Q}^i (s,a)$.

See
Algorithm~\ref{alg:chunked_td_q_v_f} for a complete description.
We call this algorithm \emph{C-factored} in our experiment in Section~\ref{exps:k2d}.

The algorithm that we call \emph{C-default} would use the probability of the \emph{entire} next percept $\hat{P}^{\pi} (X_{t+1} | X_t)$ to decay the trace.

It is important to note that the policy $\pi (.|s)$ is the global policy.
So, in the $\epsilon$-greedy case, the policy should be computed using the global Q-value.

\begin{algorithm}[tb]
\caption{Chunked SARSA}
\label{alg:chunked_td_q_q}
\begin{algorithmic}[1]
\ENSURE Eligibility $e(s, a) = 0 \ \forall (s, a) \in \mathcal{S} \times \mathcal{A}$ at the beginning of each episode
\FOR{each $(X_t, A_t, X_{t+1}, A_{t+1})$ of the episode}
\STATE Optionally train dynamics model

\FOR{each state-action pair $(s, a)$}
\STATE $e(s, a) \gets \gamma  \hat{P}^{\pi}( X_{t+1}, A_{t+1} | X_t, A_t) e(s, a)$ 
\ENDFOR
\STATE $e({S_t, A_t}) \gets e({S_t, A_t}) + 1$ 
\STATE $\delta_t \gets  R_{t+1} + \gamma \hat{Q}(S_{t+1}, A_{t+1}) - \hat{Q}(S_t, A_t)$
\FOR{each state-action pair $(s,a)$}
\STATE $\hat{Q}(s, a) \gets \hat{Q}(s, a) + \alpha \delta_t e(s, a)$
\ENDFOR
\ENDFOR
\end{algorithmic}
\end{algorithm}

\begin{algorithm}[tb]
\caption{Chunked Expected-SARSA}
\label{alg:chunked_td_q_v}
\begin{algorithmic}[1]
\ENSURE Eligibility $e(s, a) = 0 \ \forall (s, a) \in \mathcal{S} \times \mathcal{A}$ at the beginning of each episode
\FOR{each $(X_t, A_t, X_{t+1})$ of the episode}
\STATE Optionally train dynamics model

\FOR{each state-action pair $(s, a)$}
\STATE $e(s, a) \gets \gamma  \hat{P}^{\pi}( X_{t+1} | X_t) e(s, a)$ 
\ENDFOR
\STATE $e({S_t, A_t}) \gets e({S_t, A_t}) + 1$ 
\STATE $\delta_t \gets  R_{t+1} + \gamma \hat{V}(S_{t+1}) - \hat{Q}(S_t, A_t)$, where $\hat{V}(S_{t+1}) = \sum_a \pi(a | S_{t+1}) \hat{Q}(S_{t+1}, a)$
\FOR{each state-action pair $(s,a)$}
\STATE $\hat{Q}(s, a) \gets \hat{Q}(s, a) + \alpha \delta_t e(s, a)$
\ENDFOR
\ENDFOR
\end{algorithmic}
\end{algorithm}

\begin{algorithm}[tb]
\caption{Chunked Expected-SARSA with decomposed rewards}
\label{alg:chunked_td_q_v_f}
\begin{algorithmic}[1]
\ENSURE Eligibility $e^i(s, a) = 0 \ \forall (i, s, a) \in d \times \mathcal{S} \times \mathcal{A}$ at the beginning of each episode
\FOR{each $(X_t, A_t, X_{t+1})$ of the episode}
\STATE Optionally train dynamics model

\FOR{each component-state-action $(i, s, a)$}
\STATE $e^i(s, a) \gets \gamma  \hat{P}^{\pi}( X^{i}_{t+1} | X_t) e^i(s, a)$ 
\ENDFOR
\FOR{each component $i$}
\STATE $e^i({S_t, A_t}) \gets e^i({ S_t, A_t}) + 1$
\STATE $\delta^i_t \gets  R^i_{t+1} + \gamma \hat{V}^i(S_{t+1}) - \hat{Q}^i(S_t, A_t)$, where $\hat{V}^i(S_{t+1}) = \sum_a \pi(a | S_{t+1}) \hat{Q}^i(S_{t+1}, a)$

\FOR{each state-action pair $(s,a)$}
\STATE $\hat{Q}^i(s, a) \gets \hat{Q}^i(s, a) + \alpha \delta^i_t e^i(s, a)$
\ENDFOR

\ENDFOR
\ENDFOR
\end{algorithmic}
\end{algorithm}

\FloatBarrier
\section{Algorithms from Sutton and Singh 1994}
\label{app:sns94}

Here, we re-derive the approaches proposed by~\citet{sutton1994step} in order to compare them with the Chunked-TD approach proposed in this work. We modify the approaches slightly to include a reward at each step since~\citet{sutton1994step} assume the rewards only come at the end of the episode.

The basic goal of all the approaches highlighted by~\citet{sutton1994step} is to maintain 
\begin{equation}
\label{desired_condition}
    \hat{V}_{t^\prime}(s)=\frac{1}{n_{t^\prime}(s)}\left(\sum_{t< t^\prime:S_t=s}R_{t+1}+\hat{V}_{\tau(t,t^\prime)}(S_{t+1})\right)
\end{equation}
where $\hat{V}_{t^\prime}(s)$ represents the value estimate for $s$ after accounting for the transition observed at $t^\prime$, $n_{t^\prime}(s)$ is the total number of times $s$ has been visited, and $\tau(t,t^\prime)$ represents some time future time compared to $t$ to be defined. The approaches differ only in $\tau(t,t^\prime)$, that is, which time step's value estimator is used to evaluate the successor states. 
Equation~\ref{desired_condition} is only enforced at the end of episodes for those states $s$ which occur within each episode. 
The values of states which do not occur within an episode are left unchanged. 
Note that here, we use a shared time index across all episodes, so if an episode ends at time $t$, the next episode will begin at $t+1$.
We will assume acyclic MDPs as in~\citet{sutton1994step}.

One can imagine we will generally get a better estimator by pushing $\tau(t, t^\prime)$ as far into the future as possible relative to $t$ as this means the values used in the right-hand side of Equation~\ref{desired_condition} will be based on more recent information. However, doing so will also come at the price of algorithmic complexity.

\subsection{Naive TD(0)}
The simplest approach is just to use $\tau(t,t^\prime)=t+1$ and thus $\hat{V}_{t+1}(S_{t+1})$, that is, the value estimate at the time $S_{t+1}$ was entered. This can be achieved algorithmically with the following simple update on each transition:
\begin{align}
    n(S_t)&\leftarrow n(S_t)+1\\
    \hat{V}(S_t)&\leftarrow \hat{V}(S_t)+\frac{1}{n(S_t)}(R_{t+1}+\hat{V}(S_{t+1})-\hat{V}(S_t)),
\end{align}
which just computes an incremental average of the value estimates and rewards observed directly following each state $S$. This is easily seen to be $TD(0)$ with a learning rate of $\frac{1}{n(S_t)}$ and we can identify the target as $G_t^\lambda=R_{t+1}+\hat{V}(S_{t+1})$. This, however, suffers from only backing up reward information one step at a time.

\subsection{TD(1/n)}
A better approach might be to use the value estimates computed at the end of the episode in the update for each state, effectively equivalent to updating value functions in reverse order from the end of the episode to the start. We can write this as $\tau(t,t^\prime)=T(t)$ where $T(t)$ represents the termination of the episode which is in progress at time $t$ and we define $\hat{V}(\bot)=\hat{V}_{t^\prime}(\bot)=0$ for the terminal state $\bot$ at all times $t^\prime$. 

To express this algorithmically, let's define $\Tilde{V}_t$ to be the value we will set $\hat{V}_{T(t)}(S_t)$ to at the termination time $T(t)$. Note that this is different from the $\lambda$-return, toward which $\hat{V}(S_t)$ is normally updated only incrementally. With this definition, we can express the approach as follows:
\begin{align}
    n(S_t)&\leftarrow n(S_t)+1
\end{align}
\begin{align}
\label{eqn:lambda_one_over_n_update}
   \Tilde{V}_t&\leftarrow \hat{V}(S_t)+\frac{1}{n(S_t)}(R_{t+1}+\Tilde{V}_{t+1}-\hat{V}(S_t)).
\end{align}
Again, this is just performing an incremental average, but this time using the final value target $\Tilde{V}_t$ for each state rather than its value at the time the state is visited.~\footnote{Note that the assumption of acyclic MDPs avoids possible circular reference here if a particular state occurs multiple times in an episode.} We can identify the $\lambda$-return which acts as the update target for $\hat{V}(S_t)$ as $G_t^\lambda=R_{t+1}+\Tilde{V}_{t+1}$. We can then derive a recursive expression for $G_t^\lambda$ as follows
\begin{align*}
\Tilde{V}_t&=\hat{V}(S_t)+\frac{1}{n(S_t)}(G_t^\lambda-\hat{V}(S_t))\\
G_t^\lambda&=R_{t+1}+\Tilde{V}_{t+1}\\
\implies G_t^\lambda&=R_{t+1}+\frac{1}{n(S_{t+1})}G_{t+1}^\lambda+\frac{n(S_{t+1})-1}{n(S_{t+1})}\hat{V}(S_{t+1}).
\end{align*}
So we have in this case $\lambda_t=\frac{1}{n(S_{t+1})}$.

\paragraph{Equivalence between TD($1/n$) and corresponding $\lambda$-return}
We can confirm that the TD$(1/n)$ algorithm (Algorithm~\ref{alg:state_count_td_lambda}) achieves the update of Equation~\ref{eqn:lambda_one_over_n_update} in an incremental manner.

We consider the update term $u_t$ from Equation~\ref{eqn:lambda_one_over_n_update} for state $S_t$,
\begin{align*}
    u_t = \Tilde{V}_t - \hat{V}(S_t) &= \lambda_t(G_t^\lambda-\hat{V}(S_t)),
\end{align*}
where $\lambda_t = 1/n(S_t)$. Substituting the recursive relation for $G_t^{\lambda}$, we get
\begin{align*}
    u_t &= \lambda_t( R_{t+1} + \lambda_{t+1} G_{t+1}^{\lambda} + (1-\lambda_{t+1}) \hat{V}(S_{t+1}) - \hat{V}(S_{t})).
\end{align*}
We can add and subtract \textcolor{blue}{$\lambda_{t+1} \hat{V}(S_{t+1})$} inside the bracket, to get the following,
\begin{align*}
    u_t &= \lambda_t( R_{t+1} + (1-\lambda_{t+1}) \hat{V}(S_{t+1}) \textcolor{blue}{+ \lambda_{t+1} \hat{V}(S_{t+1})} - \hat{V}(S_{t})
    + \lambda_{t+1} G_{t+1}^{\lambda} \textcolor{blue}{- \lambda_{t+1} \hat{V}(S_{t+1})}), \\
    u_t &= \lambda_t(\delta_t + \lambda_{t+1}(G_{t+1}^{\lambda} - \hat{V}(S_{t+1}))),
\end{align*}
where $\delta_t = R_{t+1} + \hat{V}(S_{t+1}) - \hat{V}(S_t)$ is the standard TD-error.

Unrolling further, we obtain that the value of $S_t$ is updated as
\begin{align}
\label{eqn:total_update_for_td_one_over_n}
    u_t = \sum_{k=t}^{T-1} \delta_k \prod_{i=t}^{k} \lambda_i.
\end{align}

Now, we can verify that this is exactly the update TD($1/n$) achieves by the end of the episode.

Since the MDP is assumed to be acyclic, an encountered state is never seen again in that episode. Once a state is encountered, it is `eligible' for all future TD-errors.
For a transition at time step $t$, $S_t \rightarrow S_{t+1}$, the immediate update for $S_t$ is $\lambda_t \delta_t$, since $e(S_t)=1$ (see Line 6 of Algorithm~\ref{alg:state_count_td_lambda}).
At the next time step ($t+1$), $e(S_t) = \lambda_t$, and the update for $S_t$ would be $\lambda_t \lambda_{t+1} \delta_{t+1}$.
Thus, the sum of updates for $S_t$ by the end of the episode is $\sum_{k=t}^{T-1} \delta_k \prod_{i=t}^{k} \lambda_i$, the same as Equation~\ref{eqn:total_update_for_td_one_over_n}.

\begin{algorithm}[t]
\caption{TD$(1/n)$}\label{alg:state_count_td_lambda}
\begin{algorithmic}[1]
\ENSURE Trace $e(i) = 0 \ \forall i$ at the beginning of each episode
\FOR{each step of the episode}
\STATE $n(S_t) \gets n(S_t) + 1$
\STATE $e({S_t}) \gets 1$
\STATE $\delta_t \gets R_{t+1} + \hat{V}(S_{t+1}) - \hat{V}(S_t)$
\FOR{each state $s$}
\STATE $\hat{V}(s) \gets \hat{V}(s) + \frac{1}{n(S_t)} \delta_t e(s)$
\STATE $e(s) \gets e(s)\frac{1}{n(S_t)}$
\ENDFOR
\ENDFOR
\end{algorithmic}
\end{algorithm}

\subsection{TDC}
We can take the approach of using future estimates even further by retroactively correcting the value estimates used in past episodes each time a particular transition is revisited. We can express this as $\tau(t,t^\prime)=T(t^\prime)$. This means rather than simply using value estimates at the end of the episode in which a particular successor state was visited, we retroactively correct past value estimates used in Equation~\ref{desired_condition} to be equal to those at the end of the current episode for each visited state $s$. It now becomes a little more involved to formulate incrementally. In particular, we will have to maintain some extra information.

Towards formulating an incremental update, define $n(s^\prime,s)$ as the number of times the transition $(s^\prime,s)$ has been visited. Also, define $\hat{V}(s^\prime,s)=\hat{V}_{T(s^\prime,s)}(s^\prime)$ where $T(s^\prime,s)$ is the termination time of the most recent episode in which the transition from $s$ to $s^\prime$ was visited. The reason for maintaining these quantities is essentially so we can subtract the older estimate $\hat{V}(s^\prime,s)$ from the estimator $\hat{V}(s)$ in order to replace it with an updated value each time the transition is visited. With these definitions in place, we are ready to define 
$\Tilde{V}_t$ which, as in the previous section, represents the value we will set $\hat{V}_{T(t)}(S_t)$ to at the termination time $T(t)$.
\begin{align}
    n(S_{t+1},S_t) &\leftarrow n(S_{t+1},S_t) + 1 \\
    n(S_t) &\leftarrow n(S_t) + 1 \\
    \Tilde{V}_t &  \!\begin{multlined}[t]
    \leftarrow\frac{n(S_t) - 1}{n(S_t)}\hat{V}(S_t) - \frac{n(S_{t+1},S_t) - 1}{n(S_t)}\hat{V}(S_{t+1},S_t) \\
    + \frac{n(S_{t+1},S_t)}{n(S_t)}\Tilde{V}_{t+1} + \frac{1}{n(S_t)}R_{t+1}.
    \end{multlined}
\end{align}
To see how these updates enforce Equation~\ref{desired_condition}, assume that at time $t-1$ Equation~\ref{desired_condition} holds for $S_t$. Then, at time $t$, we visit $S_t$ and wish to make the correction for the new visit. Because we maintain $\hat{V}(s^\prime,s)=\hat{V}_{T(s^\prime,s)}(s^\prime)$, we can write
\begin{equation}
    \hat{V}(S_t)=\bar{R}+\frac{1}{n(S_t)-1}\left(\sum_{s^\prime\neq S_t}n(s^\prime,S_t)\hat{V}(s^\prime,S_t)+(n(S_{t+1},S_t)-1)\hat{V}(S_{t+1},S_t)\right),
\end{equation}
at the time $t$ for which we are computing $\Tilde{V}_t$, where $\bar{R}$ is the average of past rewards following state $S_t$. Substituting this into the expression for $\Tilde{V}_t$ we get:
\begin{align}
\Tilde{V}_t&=\begin{multlined}[t]
\frac{1}{n(S_t)}\Bigl(\sum_{s^\prime\neq S_t}(n(s^\prime,S_t)-1)\hat{V}(s^\prime,S_t)+(n(S_{t+1},S_t)-1)\hat{V}(S_{t+1},S_t)\\
-(n(S_{t+1},S_t) - 1)\hat{V}(S_{t+1},S_t)+n(S_{t+1},S_t)\Tilde{V}_{t+1}+(n(S_t)-1)\bar{R}+R_{t+1}\Bigr).
\end{multlined}
\end{align}
In effect, the update subtracts the outdated estimate from the sum, replaces it with the newer estimate, updates the incremental average for the reward, and corrects the denominator to account for the additional visit to $S_t$. Having defined $\Tilde{V}_t$, we can now try to rewrite it in a form that looks more like a conventional TD update
\begin{equation}
\Tilde{V}_t = \hat{V}(S_t)+\frac{1}{n(S_t)}(R_{t+1}+n(S_{t+1},S_t)\Tilde{V}_{t+1}-(n(S_{t+1},S_t) - 1)\hat{V}(S_{t+1},S_t)-\hat{V}(S_t)).
\end{equation}
The analogy to the lambda return in this case is
\begin{equation}
G_t^\lambda = R_{t+1}+n(S_{t+1},S_t)\Tilde{V}_{t+1}-(n(S_{t+1},S_t) - 1)\hat{V}(S_{t+1},S_t).
\end{equation}
Which we can once again put in a recursive form
\begin{align}
G_t^\lambda &= R_{t+1}+n(S_{t+1},S_t)\Tilde{V}_{t+1}-(n(S_{t+1},S_t) - 1)\hat{V}(S_{t+1},S_t)\\
\Tilde{V}_t &= \hat{V}(S_t)+\frac{1}{n(S_t)}(G_t^\lambda-\hat{V}(S_t))\\
\implies G_t^\lambda &=\begin{multlined}[t]R_{t+1}+n(S_{t+1},S_t)\frac{n(S_{t+1})-1}{n(S_{t+1})}\hat{V}(S_{t+1})-\\(n(S_{t+1},S_t) - 1)\hat{V}(S_{t+1},S_t)+\frac{n(S_{t+1},S_t)}{n(S_{t+1})}G_{t+1}^\lambda
\label{eqn:recursive_form_tdc_update}
\end{multlined}
\end{align}
This looks like a $\lambda$-return with $\lambda=\frac{n(S_{t+1},S_t)}{n(S_{t+1})}$, except that instead of $(1-\lambda)\hat{V}(S_{t+1})$ we have a rather elaborate difference term. However, note that in the case where $\hat{V}(S_{t+1},S_t)=\hat{V}(S_{t+1})$ this extra term reduces to simply $\left(1-\frac{n(S_{t+1},S_t)}{n(S_{t+1})}\right)\hat{V}(S_{t+1})$. Thus, as long as $\hat{V}(S_{t+1},S_t)$ is up to date, this is indeed just a normal $\lambda$-return with $\lambda=\frac{n(S_{t+1},S_t)}{n(S_{t+1})}$. Otherwise, it uses some extra machinery to subtract the outdated value estimate from the current target before adding the new one.

Note that $\frac{n(S_{t+1},S_t)}{n(S_{t+1})}$ is the transition probability under an empirical \textit{backward} model. Hence the TDC algorithm essentially corresponds to setting lambda based on estimated backward transition probabilities. 
One could also write this in terms of forward model probabilities as $\frac{n(S_{t+1},S_t)}{n(S_{t})}\frac{n(S_t)}{n(S_{t+1})}$ by multiplying by the visitation ratio.
As we show in the next section, the online backward view of the algorithm can be implemented using a forward model.
The difference arises from whether the learning rate for the update is set to be equal to the count for the state being updated itself or the future state at which the update actually occurs in the backward view.
In Chunked-TD we use a constant learning rate, so the correspondence is not exact.

\begin{algorithm}[t]
\caption{TDC}\label{alg:corrected_td}
\begin{algorithmic}[1]
\ENSURE Trace $e(i) = 0 \ \forall i$ at the beginning of each episode
\FOR{each step of the episode}
\STATE $n(S_{t+1}, S_t) \gets n(S_{t+1}, S_t) + 1$
\STATE $n(S_t) \gets n(S_t) + 1$
\STATE $e({S_t}) \gets 1$ 
\STATE $\delta'_t \gets  R_{t+1} + \hat{V}(S_{t+1}) + (n(S_{t+1}, S_t) - 1) [\hat{V}(S_{t+1}) - \hat{V}(S_{t+1}, S_t)]  - \hat{V}(S_t)$
\FOR{each state $s$}
\STATE $\hat{V}(s) \gets \hat{V}(s) + \frac{1}{n(S_t)} \delta'_t e(s)$
\STATE $e(s) \gets e(s)\frac{n(S_{t+1}, S_t)}{n(S_t)} $
\ENDFOR
\ENDFOR
\FOR{each non-terminal $S_{t+1}$ in the episode}
\STATE $\hat{V}(S_{t+1}, S_t) = \hat{V}(S_{t+1})$
\ENDFOR
\end{algorithmic}
\end{algorithm}

\paragraph{Equivalence between TDC and corresponding $\lambda$-return}

Proceeding similarly to our analysis for TD$(1/n)$ in the previous section, let us consider the update ($u_t$) made by our recursive lambda return from Equation~\ref{eqn:recursive_form_tdc_update},
\begin{align*}
    u_t = \Tilde{V}_t - \hat{V}(S_t) &= \frac{1}{n(S_t)}(G_t^\lambda-\hat{V}(S_t)).
\end{align*}

We use the same trick of expanding $G_t^{\lambda}$ recursively, to get
\begin{align*}
    u_t = \frac{1}{n(S_t)}(R_{t+1} + n(S_{t+1}, S_t) \left ( 1 - \frac{1}{n(S_{t+1})}\right ) \hat{V} (S_{t+1}) - \left (n(S_{t+1}, S_t) - 1 \right) \hat{V} (S_{t+1}, S_t) \\
    + \lambda_{t+1} G_{t+1}^{\lambda} -
    \hat{V}(S_t)),
\end{align*}
where $\lambda_{t+1} = \frac{n(S_{t+1}, S_t)}{n(S_{t+1})}$.

The above can be re-expressed as
\begin{align*}
    u_t = \frac{1}{n(S_t)}(R_{t+1} +  \hat{V} (S_{t+1}) + \left (n(S_{t+1}, S_t) - 1 \right) \left( \hat{V} (S_{t+1}) - \hat{V} (S_{t+1}, S_t) \right )\\
    + \lambda_{t+1} G_{t+1}^{\lambda} - \lambda_{t+1} \hat{V} (S_{t+1}) - 
    \hat{V}(S_t)).
\end{align*}
Using the definition of the corrected-TD error ($\delta'_t$) from Algorithm~\ref{alg:corrected_td},
\begin{align*}
 u_t = \frac{1}{n(S_t)} \left ( \delta'_t + \lambda_{t+1} (G_{t+1}^{\lambda} - \hat{V} (S_{t+1})) \right )   
\end{align*}
Unrolling the recursion, we get
\begin{align}
\label{eqn:total_update_for_tdc}
 u_t = \frac{1}{n(S_t)} \sum_{k=t}^{T-1} \delta'_k \prod_{i=t+1}^{k} \lambda_i.
\end{align}

Now let us consider the sequence of updates for a state $S_t$ under Algorithm~\ref{alg:corrected_td}. 

The first term of the update $u_t^1$ is $$u_t^{1} = \frac{1}{n(S_t)} \cdot \delta'_t \cdot 1$$
These updates are written as update = learning rate $\cdot$ corrected-td-error $\cdot$ trace value of $S_t$.

Now, in the next update the learning rate is $1/n(S_{t+1})$, and the trace has been decayed by $n(S_{t+1}, S_t )/n(S_{t})$, giving us
$$u_t^{2} = \frac{1}{n(S_{t+1})} \cdot \delta'_{t+1} \cdot \frac{n(S_{t+1}, S_t )}{n(S_{t})}$$
Similarly,
$$u_t^{3} = \frac{1}{n(S_{t+2})} \cdot \delta'_{t+2} \cdot \frac{n(S_{t+1}, S_t )}{n(S_{t})} \cdot \frac{n(S_{t+2}, S_{t+1} )}{n(S_{t+1})}$$
Summing all the updates to and taking $1/n(S_t)$ common, we get $ \frac{1}{n(S_t)} \sum_{k=t}^{T-1} \delta'_k \prod_{i=t+1}^{k} \lambda_i$, which matches the update from Equation~\ref{eqn:total_update_for_tdc}.

\end{document}